\UseRawInputEncoding
\documentclass[10pt,twocolumn,letterpaper]{article}
\usepackage[pagenumbers]{cvpr} 

\usepackage{graphicx}
\usepackage{amsmath}
\usepackage{amssymb}
\usepackage{mathtools}
\usepackage{amsthm}
\usepackage{booktabs}
\usepackage{caption}
\usepackage{subcaption}
\usepackage{multirow}
\usepackage{multicol}
\usepackage{stfloats}
\usepackage{bm}
\newtheorem{theorem}{Theorem}[section]

\newtheorem{lemma}[theorem]{Lemma}

\usepackage{algorithm}
\usepackage{algpseudocode}

\usepackage[pagebackref,breaklinks,colorlinks]{hyperref}

\usepackage[capitalize]{cleveref}
\crefname{section}{Sec.}{Secs.}
\Crefname{section}{Section}{Sections}
\Crefname{table}{Table}{Tables}
\crefname{table}{Tab.}{Tabs.}

\def\cvprPaperID{6743} 
\def\confName{CVPR}
\def\confYear{2023}

\begin{document}

\title{Adapting Step-size: A Unified Perspective to Analyze and Improve Gradient-based Methods for Adversarial Attacks}

\newcommand*\samethanks[1][\value{footnote}]{\footnotemark[#1]}
\author{Wei Tao\thanks{Equal contribution}\\
Academy of Military Science\\
{\tt\small wtao\_plaust@163.com}
\and
Lei Bao\samethanks,  Qing Tao\thanks{Corresponding author}\\
Army Academy of Artillery and Air Defense\\
{\tt\small baolei1213@sina.com, taoqing@gmail.com}
\and
Sheng Long\\
National University of Defense Technology\\
{\tt\small longsheng@nudt.edu.cn}
\and
Gaowei Wu\\
Chinese Academy of Sciences\\
{\tt\small gaowei.wu@ia.ac.cn}
}
\maketitle

\begin{abstract}
Learning adversarial examples can be formulated as an optimization problem of maximizing the loss function with some box-constraints. However, for solving this induced optimization problem, the state-of-the-art gradient-based methods such as FGSM, I-FGSM and MI-FGSM look different from their original methods especially in updating the direction, which makes it difficult to understand them and then leaves some theoretical issues to be addressed in viewpoint of optimization. In this paper, from the perspective of adapting step-size, we provide a unified theoretical interpretation of these gradient-based adversarial learning methods. We show that each of these algorithms is in fact a specific reformulation of their original gradient methods but using the step-size rules with only current gradient information. Motivated by such analysis, we present a broad class of adaptive gradient-based algorithms based on the regular gradient methods, in which the step-size strategy utilizing information of the accumulated gradients is integrated. Such adaptive step-size strategies directly normalize the scale of the gradients rather than use some empirical operations. The important benefit is that convergence for the iterative algorithms is guaranteed and then the whole optimization process can be stabilized. The experiments demonstrate that our AdaI-FGM consistently outperforms I-FGSM and AdaMI-FGM remains competitive with MI-FGSM for black-box attacks.
\end{abstract}

\section{Introduction}
\label{sec:intro}
For quite a few years, deep learning has achieved widespread empirical success in many fields especially in computer vision \cite{krizhevsky2012imagenet} and natural language procession \cite{Devlin2019BERTPO}. In spite of this success, evidence has shown that deep neural networks are challenged by their vulnerability to adversarial attacks \cite{Szegedy2014IntriguingPO}, where adversaries craft a fraction of the training data by adding human-imperceptible perturbations to subvert the learning process. 
One interesting phenomenon is that adversarial examples can often transfer from one model to another. The fact behind this transferability is that different learning models learn almost similar classifiers, making it possible to attack models without knowing their structure and parameters \cite{Szegedy2014IntriguingPO, Dong2018BoostingAA}.

For a given learning model, generating adversarial examples can be formulated as an optimization problem of maximizing the loss function with box-constraints. To solve this induced optimization problem, many gradient-based methods have been developed. Typical instances include FGSM \cite{Goodfellow2015ExplainingAH}, I-FGSM \cite{Kurakin2017AdversarialEI}, MI-FGSM \cite{Dong2018BoostingAA} and NI-FGSM \cite{Lin2020NesterovAG}. Among them, FGSM and I-FGSM are directly motivated by the gradient descent strategy with one or more iterative steps, and MI-FGSM and NI-FGSM are established upon Polyak's heavy-ball (HB) method \cite{Polyak1964SomeMO} and Nesterov's accelerated gradient (NAG) method \cite{nesterov27method}.

As pointed out in \cite{Dong2018BoostingAA}, momentum, which utilize the accumulation of past gradients, can stabilize the update directions and escape from poor local maxima. In practice, it has been shown that the adversarial examples generated by MI-FGSM \cite{Dong2018BoostingAA} and NI-FGSM \cite{Lin2020NesterovAG} have higher success rates in both white-box and black-box attacks. Specifically, with MI-FGSM, they won the first places in NIPS 2017 Non-targeted Adversarial Attack and Targeted Adversarial Attack competitions \cite{Dong2018BoostingAA}. Such a fact further reveals that different optimization methods may have different effects on some learning tasks, which has already been shown in sparse learning problems \cite{Xiao2009DualAM}.

While it now becomes a common viewpoint that the field of adversarial attacks and defenses is dominated by gradient-based approaches, there are still some theoretical problems worthy of further study. First, once an optimization algorithm is concerned, its convergence, which provides theoretical guarantees for stability of the whole iterative process, may be one of the most important issues. The convergence analysis of a gradient-based algorithm heavily depends on the selection of its update direction and step-size rule \cite{bertsekas2003convex}. Nevertheless, FGSM empirically uses the gradient-sign as its update direction rather than the real gradient. Despite the fact that using sign function can normalize the scale of the gradients and restrict generated adversarial example to be within the constrained domain, FGSM and I-FGSM solve the constrained problems in a heuristic manner and they no longer have the property of greedy ascent. Inevitably, the convergence of I-FGSM needs to be restudied. Secondly, besides continuing to use the sign of past gradients accumulation as its iterative directions, both MI-FGSM and NI-FGSM employ $l_1$-norm to normalize the scale of the gradients in different iterations that varies in magnitude. Due to these differences from the regular momentum methods, even for convex objective functions, it becomes unclear if the benefit of accelerated convergence \cite{nesterov27method, Tao2020TheSO, Tao2021TheRO} and  avoiding local minima \cite{sun2019heavy} still be kept. Finally, note that each gradient-based method for adversarial attacks is a variant of one regular optimization algorithm. However, so far, there still lacks explicit analysis of the relationship between these variants and the induced optimization problems.

From the perspective of optimization, several theoretical issues should be argued. First of all, \emph{ what is the relationship between gradient-based attack methods and their original gradient algorithms?} At the first sight, they even have different update directions. On the other hand, due to the fact that the attack learning is described as an optimization problem, \emph{why don't we directly employ the gradient or momentum methods?} Furthermore, note that the sign function plays an important part in the normalization. However, this operation looks empirical in regular optimization algorithms. \emph{Can we have some other way to normalize the scale of gradients while keeping convergence of the algorithms?} Finally, as information of past gradients is only used to stabilize the update directions, \emph{ can we use such information to further stabilize the whole process of optimization?} The motivation of this paper is to address these issues with a unified perspective and our key idea is to utilize the adaptive step-size strategy.

Adaptive step-size have proved effective in overcoming the limitation of the optimization methods that each element in the update direction is uniformly scaled. The first algorithm along this line is AdaGrad \cite{Duchi2010AdaptiveSM}. By using a sum of squared past gradient values, AdaGrad re-scales each coordinate of the gradient. What is more, AdaGrad enjoys the same convergence rate as vanilla SGD but with a smaller factor in sparse learning problems. Unfortunately, experiments have demonstrated that it under-performed when applied to the training deep neural networks \cite{Wilson2017TheMV}. This is because the large impact of past gradients in the adaptive strategy prevents it from adapting to local changes in the smoothness of the function. Practical experience has led to the development of adaptive methods that is able to emphasize the more recent gradients. To do this, RMSProp \cite{tieleman2012lecture} uses an exponential moving average (EMA) to replace a cumulative sum to forget past gradients. So far, EMA has been regarded as a commonly-used technique to deal with accumulated information. Adam \cite{Kingma2015AdamAM}, currently one of the most popular training algorithms in deep neural networks, uses EMA in the update of both step-size and direction.

There have been several reports about the direct application of adaptive gradient-based methods in solving adversarial attack optimization problems. For example, a variant of Adam was proposed to generate indistinguishable adversarial examples with high transferability \cite{Zou2022MakingAE}. In \cite{Yang2022AdversarialEG}, Adabelief optimizer \cite{Zhuang2020AdaBeliefOA} is introduced to improve the transferability of adversarial examples. In this paper, we first use specific step-size strategies to analyze the connection between gradient-based attack methods and their original gradient algorithms. Then, we obtain a broad class of adaptive gradient-based adversarial learning algorithms by adapting the step-size with information of the accumulated gradients. In contrast to the investigations in \cite{Zou2022MakingAE} and \cite{Yang2022AdversarialEG}, the adaptive methods here (Section 3.2) are established upon our analysis (Section 3.1). The contributions in this paper can be summarized as follows,
\begin{itemize}
\item  We provide a unified theoretical interpretation of FGSM, I-FGSM and MI-FGSM from the viewpoint of adapting step-size. We show that each of these algorithms is in fact a specific reformulation of their original gradient methods but using the step-size rules with only current gradient information.
\item We present AdaI-FGM and AdaMI-FGM to improve the available algorithms, in which the accumulated gradients is further used to adapt the step-sizes of regular gradient-based methods. It not only normalizes the scale of the gradients but also stabilizes the whole optimization process. The experiments demonstrate that adapting step-size with the accumulated gradients can remarkably improves the adversarial attacks.
\item The experiments also illustrate that the derived AdaMI-FGM remains competitive with MI-FGSM for black-box attacks. This will inspire us to introduce more adaptive and regular optimization algorithms to the field of adversarial attacks. We hope that the proposed adaptive step-size strategy can serve as a general and effective technique to boost the transferability and stability of available gradient-based methods.
\end{itemize}

\section{Related work}
\label{sec:Related Work}
In this section, we first describe the optimization problem for adversarial attacks and then provide a brief overview of several typical gradient-based attack methods and adaptive gradient methods.

\subsection{Optimization problems}

Let $\mathcal{S}=\{(\bm{x}_1, y_1), \dots, (\bm{x}_m, y_m)\}$ be a training set, where $y_i$ is the label of $\bm{x}_i \in \mathbb{R}^{d} $ and the sample $(\bm{x}_{i},y_{i})$ is uniformly random chosen from a distribution $\mathcal{D}$.

Adversarial training can be formulated as a {\it min-max} optimization problem \cite{Madry2018TowardsDL, Wang2019OnTC}
\begin{equation}\label{adav-training}
\min_{\bm{\theta}}\mathbb{E}_{(\bm{x},y)\thicksim \mathcal{D}} \max_{\bm{x}^{adv} \in \mathcal{B}_{\epsilon}(\bm{x})} J(\bm{x}^{adv},y),
\end{equation}
where $f_{\bm{\theta}}(\bm{x}):\bm{x} \in \bm{X} \subseteq \mathbb{R}^{d} \rightarrow y \in Y$ is a classifier with parameters $\bm{\theta}$, $\mathcal{B}_{\epsilon}(\bm{x})=\{\bm{x}^{adv}: \|\bm{x}^{adv}-\bm{x}\|_p \leq \epsilon\}$ and $J(\bm{x},y)$ is the loss function. Throughout this paper, we use the cross-entropy loss and assume that $J(\bm{x},y)$ is differentiable $w. r. t.$ $\bm{x}$.

In contrast to adversarial training, on the attack side, we are given a classifier $f_{\bm{\theta}}$ with a predefined $\bm{\theta}$. Generating a non-targeted adversarial example $\bm{x}^{adv}$ from
a real example $\bm{x}$ can be formulated as the following optimization problem \cite{Madry2018TowardsDL, Goodfellow2015ExplainingAH},
\begin{equation}\label{adv-optimization}
\max J(\bm{x}^{adv},y), \ s. t. \  \|\bm{x}^{adv}-\bm{x}\|_p \leq \epsilon.
\end{equation}

Obviously, optimization problem (\ref{adv-optimization}) coincides with our intuition, i.e., adversarial attack is to find an example $\bm{x}^{adv}$ that misleads the model prediction $ f_{\bm{\theta}}(\bm{x}^{adv})\neq y$ while the $l_p$-norm of the adversarial perturbation $\|\bm{x}^{adv}-\bm{x}\|_p$ should be restricted to a threshold $\epsilon$. Alternatively, learning adversarial examples can also be described as a {\it regularized} optimization problem \cite{Carlini2017TowardsET}
$$
\min \lambda\|\bm{x}^{adv}-\bm{x}\|_p- J(\bm{x}^{adv},y),
$$
where $\lambda$ is the trade-off parameter. In this paper, we only focus on adversarial attack optimization problem (\ref{adv-optimization}) with $p=\infty$.

\subsection{Gradient-based attack methods}
To solve optimization problem (\ref{adv-optimization}), many gradient-based attack methods have been developed.

FGSM \cite{Goodfellow2015ExplainingAH} is one of the most basic gradient-based attack methods. It has only one-step update, which is described as follows,
\begin{equation}\label{FGSM}
\bm{x}^{adv}=\bm{x}+\epsilon \ \emph{sign}(\nabla_{\bm{x}} J(\bm{x},y)),
\end{equation}
where $\nabla_{\bm{x}}J(\bm{x},y)$ is the gradient of $J(\bm{x},y)$ $w. r. t.$ $\bm{x}$ and $sign(\cdot)$ is the sign function. From (\ref{FGSM}), it is easy to find $\|\bm{x}^{adv}-\bm{x}\|_\infty \leq \epsilon$.

I-FGSM \cite{Kurakin2017AdversarialEI} is in fact a FGSM with multiple iterative steps, which key operation is
\begin{align}
\bm{x}_{t+1}^{adv}= \bm{x}_{t}^{adv} + \alpha \ \emph{sign}(\nabla_{\bm{x}} J(\bm{x}_{t}^{adv},y))\label{I-FGSM},
\end{align}
where $\bm{x}^{adv}_0 = \bm{x}$. Unlike FGSM, some operations should be added in I-FGSM to restrict each $\bm{x}_{t}^{adv}$ to satisfy $\|\bm{x}_{t}^{adv}-\bm{x}\|_\infty \leq \epsilon$. To do this, one can either set the step-size $\alpha=\epsilon/T$ with $T$ being the total number of iterations or use a $Clip_{\bm{x}}^\epsilon\{\cdot\}$ function. Specifically, for an image $\bm{x}^{adv}=(x^{adv}_1,x^{adv}_2,x^{adv}_3)$ which is typically 3-D tensor, its clip operation is \cite{Kurakin2017AdversarialEI}
\begin{equation}\label{clip}
\begin{aligned}
& Clip_{\bm{x}}^\epsilon (\bm{x}^{adv}(x^{adv}_1,x^{adv}_2,x^{adv}_3))\\
& =\min\{255, \bm{x}(x_1,x_2,x_3)+\epsilon,\\
& max\{0, \bm{x}(x_1,x_2,x_3)-\epsilon, \bm{x}^{adv}(x^{adv}_1,x^{adv}_2,x^{adv}_3)\}\}.
\end{aligned}
\end{equation}

Compared with FGSM, I-FGSM has a higher success rate for white-box attacks but at the cost of worse transferability, which means that I-FGSM is less effective in black-box environments \cite{Kurakin2017AdversarialEI}.

MI-FGSM \cite{Dong2018BoostingAA} integrates HB momentum \cite{Polyak1964SomeMO} into the iterative step of I-FGSM. The update procedure of MI-FGSM is
\begin{equation}\label{MI}
\left  \{
\begin{array}{l}
\bm{g}_{t+1}=\mu \ \bm{g}_t  +  \frac{\nabla_{\bm{x}} J(\bm{x}_{t}^{adv},y)}{\|\nabla_{\bm{x}} J(\bm{x}_{t}^{adv},y)\|_1}\\
\bm{x}_{t+1}^{adv}=Clip_{ \bm{x}}^\epsilon  \{\bm{x}_{t}^{adv} + \alpha \ \emph{sign}(\bm{g}_{t+1})\}
\end{array},
\right.
\end{equation}
where $\mu$ is the decay factor with $\bm{g}_0=0$.

Similarly, the NAG momentum \cite{nesterov27method} can also be intergraded into I-FGSM. The iterative version of NI-FGSM \cite{Lin2020NesterovAG} is formulated as
\begin{equation}\label{NI}
\left  \{
\begin{array}{l}
\bm{x}^{nes}_t= \bm{x}^{adv}_t  +\alpha  \bm{g}_t \\
\bm{g}_{t+1}=\mu \ \bm{g}_t  + \frac{\nabla_{\bm{x}} J(\bm{x}_{t}^{nes},y)}{\|\nabla_{\bm{x}} J(\bm{x}_{t}^{nes},y)\|_1} \\
\bm{x}_{t+1}^{adv}=Clip_{\bm{x}}^\epsilon  \{\bm{x}_{t}^{adv} + \alpha \ \emph{sign}(\bm{g}_{t+1})\}
\end{array}.
\right.
\end{equation}

Compared with FGSM and I-FGSM, the update direction $\bm{g}_{t}$ in (\ref{MI}) and (\ref{NI}) now becomes the accumulated gradients. With such accumulation, the momentum methods can stabilize their iterative direction $ \emph{sign}(\bm{g}_{t+1})$. As pointed out in \cite{Dong2018BoostingAA}, the transferability of adversarial examples is boosted.

\subsection{Adaptive gradient-based methods}

Generally speaking, projected gradient descent (PGD) is one of the most fundamental algorithms for dealing with constrained minimization problem. For solving (\ref{adv-optimization}), its iteration becomes
\begin{equation}\label{PGD}
\bm{x}^{adv}_{t+1}= P_{\mathbf{Q}}[\bm{x}^{adv}_{t} + \alpha_t \nabla_{\bm{x}} J(\bm{x}^{adv}_{t},y)],
\end{equation}
where $\alpha_t>0$ is the time-varying step-size and $P_{\mathbf{Q}}({\cdot})$ is the projection operator on the closed convex set $\mathbf{Q}=\{\bm{z} \in \mathbb{R}^{d}: \| \bm{z}-\bm{x}\|_\infty \leq \epsilon\}$ \cite{bertsekas2003convex}.

Simply speaking, AdaGrad \cite{Duchi2010AdaptiveSM, wang2019sadam} takes the form of
\begin{equation}\label{AdaGrad}
 \bm{x}^{adv}_{t+1}=P_{\mathbf{Q}}[\bm{x}^{adv}_{t}+\displaystyle\frac{\alpha} {\sqrt{t}}V_{t}^{-\frac{1}{2}} \nabla_{\bm{x}} J(\bm{x}^{adv}_{t},y)],
\end{equation}
where $V_{t}$ is a $d \times d$ diagonal matrix and
\begin{equation}\label{arithmetic average}
v_{t,i}=\frac{\sum_{j=1}^{t}{[ \nabla_{\bm{x}} J(\bm{x}^{adv}_{j},y)]_{i}^{2}}}{t}
\end{equation}
is the arithmetic average of the square of the $i$-th elements of the past gradients. Obviously, the seldom-updated weights are updated with a larger step size than the frequently-updated weights. For this reason, the adaptive mechanism is well-suited for sparse learning problems.

RMSProp \cite{tieleman2012lecture} replaces the arithmetic average procedure (\ref{arithmetic average}) in AdaGrad (\ref{AdaGrad}) with EMA, i.e.,
\begin{equation}\label{EMA}
V_{t}=\beta V_{t-1}+(1-\beta)\text{diag}(\nabla_{\bm{x}}J(\bm{x}^{adv}_{t},y) \nabla_{\bm{x}}J(\bm{x}^{adv}_{t},y)^{T}),
\end{equation}
where $diag(\cdot)$ denotes extracting the diagonal matrix and $0 \leq \beta\leq 1$. With EMA, the weights assigned to past gradients decay exponentially so that the reliance of the update is essentially limited to recent few gradients. The exponential discount factor $\beta$ controls how slowly the momentum buffer is updated.

\section{Adaptive methods for adversarial attacks}

In this section, we first analyze the connection between the gradient-based attack methods and their original gradient algorithms. Based upon such analysis, we will propose several adaptive algorithms for adversarial attacks.

\subsection{Analyzing gradient-based attack methods}

We first consider the relationship between I-FGSM (\ref{I-FGSM}) and PGD (\ref{PGD}).

We indicate that $Clip_{\bm{x}}^\epsilon(\bm{x}^{adv})$ in (\ref{I-FGSM}), (\ref{MI}) and (\ref{NI}) is a projection of $\bm{x}^{adv}$ on a specific $\mathbf{Q}$. For example, when an image is described as 3-D tensor, $\mathbf{Q}$ is naturally set to $\{\bm{z}: \| \bm{z}-\bm{x}\|_\infty \leq \epsilon\}\bigcap [0,255]^3$. Note that clipping all the coordinates of $\bm{x}^{adv}$ to be within the box $\{\bm{z}: \| \bm{z}-\bm{x}\|_\infty \leq \epsilon\}$ by one step PGD has also been discussed in \cite{Carlini2017TowardsET} and \cite{Madry2018TowardsDL}.

In general cases, $\emph{sign}(\bm{g}_{t})\neq \bm{g}_{t}$. This means that I-FGSM (\ref{I-FGSM}) and PGD (\ref{PGD}) have different update directions. However, if we set
\begin{equation}\label{not-historical}
v_{t,i}=\frac{[{ \nabla_{\bm{x}} J(\bm{x}^{adv}_{t},y)]_{i}}^{2}}{t},
\end{equation}
AdaGrad (\ref{AdaGrad}) will become I-FSGM (\ref{I-FGSM}). This reveals that I-FSGM is in fact a specific PGD but using different step-size rules. As the gradient-sign is now used in almost all the gradient-based algorithms such as FGSM and MI-FGSM, the above analysis is also applicable to these algorithms.

In contrast to resetting a specific step-size rule (\ref{not-historical}) in I-FGSM, the analysis of MI-FGSM (\ref{MI}) is more complex. This is because the momentum parameter shoud also be considered.

It should be pointed out that the regular HB momentum method \cite{Tao2021TheRO} for constrained optimization problem (\ref{adv-optimization}) is
\begin{equation}\label{constrained HB}
\bm{x}_{t+1}^{adv}= P_{\mathbf{Q}} [\bm{x}_{t}^{adv} + \alpha_t  \nabla_{\bm{x}} J(\bm{x}_{t}^{adv},y) +\mu (\bm{x}_{t}^{adv}-\bm{x}_{t-1}^{adv})].
\end{equation}

For an optimization problem without constraints, HB (\ref{constrained HB}) can be rewritten as a two-steps algorithm
\begin{equation}\label{two-step}
\left  \{
\begin{array}{l}
\bm{g}_{t+1}=\mu \ \bm{g}_t  + \alpha_t \ \nabla_{\bm{x}} J(\bm{x}_{t}^{adv},y)\\
\bm{x}_{t+1}^{adv}=\bm{x}_{t}^{adv} +  \bm{g}_{t+1}
\end{array}.
\right.
\end{equation}

HB in the form of (\ref{two-step}) is popularly used in many deep learning references such as \cite{Ruder2016AnOO}.  Although we can set $\alpha_t=\frac{1}{\|\nabla_{\bm{x}} J(\bm{x}_{t}^{adv},y)\|_1}$, the update direction of two-steps HB (\ref{two-step}) is $\bm{g}_{t+1}$, which is still different from the update direction $\emph{sign}(\bm{g}_{t+1})$ in MI-FSGM (\ref{MI}). Like that in the analysis of I-FGSM, we introduce a diagonal matrix $V_{t}$ in (\ref{two-step}) and obtain an extended version of HB (\ref{two-step})
\begin{equation}\label{two-step-2}
\left  \{
\begin{array}{l}
\bm{g}_{t+1}=\mu \ \bm{g}_t  + \alpha_t \ \nabla_{\bm{x}} J(\bm{x}_{t}^{adv},y)\\
\bm{x}_{t+1}^{adv}=\bm{x}_{t}^{adv} + V_{t} \bm{g}_{t+1}
\end{array}.
\right.
\end{equation}

We can further set $v_{t,i}=\frac{1}{|[\bm{g}_{t+1}]_{i}|}$. Now, (\ref{two-step-2}) coincides with MI-FSGM (\ref{MI}), i.e., MI-FSGM (\ref{MI}) is a specific form of an extended version of HB (\ref{two-step-2}) but using different step-size $\alpha_t$ and $V_t$ from commonly used  HB (\ref{two-step}).

However, the step-size rule (\ref{not-historical}) only employs the current gradient information $\nabla_{\bm{x}} J(\bm{x}_{t}^{adv},y)$. On the other hand, as the specific step-size rule (\ref{not-historical}) is different from that in regular algorithms, even for convex objective functions, we don't know whether (\ref{AdaGrad}) and (\ref{two-step-2}) are convergent or not. Thus, the stability of the whole optimization process can not be guaranteed. Further, if the algorithm (\ref{two-step-2}) has the property of accelerating convergence and avoiding local maxima should be restudied. Finally, we don't know if $J(\bm{x}_{t}^{adv},y)$ is convergent to $ \max J(\bm{x}^{adv},y)$ on $  \|\bm{x}^{adv}-\bm{x}\|_p \leq \epsilon$.

\subsection{AdaI-FGM and AdaMI-FGM}

Motivated by the above analysis and successful using of the accumulated gradients in MI-FGSM, we propose a broad class of adaptive gradient-based methods, in which the step-sizes are updated by using the accumulated gradients. As the adaptive step-size can normalize the scale of the gradients, the sign function will no longer be used.

Specifically, the detailed steps of our adaptive I-FGM  are shown in Algorithm \ref{alg:AdaI-FGM}.
\begin{algorithm}
\caption{AdaI-FGM}
\label{alg:AdaI-FGM}
    \begin{algorithmic}[1]
    \Require
    A target classifier $f$ with loss function $J$, a benign image $\bm{x}$ with its ground-truth label $y$.
    \Require
    The perturbation size $\epsilon$, step-size parameter $\alpha>0$, constant $\delta > 0$, EMA parameter $\beta>0$ and total number of iterations $T$.
     \State Initialize $\bm{x}^{adv}_0=\bm{x}$ and $V_0 = \mathbf{0}_{d\times d}$.
    \Repeat
    \State Update $V_t$ by Eq. (\ref{arithmetic average}) or Eq. (\ref{EMA}).
    \State $\hat{V_t} = V_{t}^{\frac{1}{2}}+\displaystyle\frac{\delta }{\sqrt{t}}I_d$.
    \State $\bm{x}_{t+1}^{adv}= P_{\mathbf{Q}} [\bm{x}_{t}^{adv} + \displaystyle\frac{\alpha}{\sqrt t} \hat{V_t}^{-1} \nabla_{\bm{x}}J(\bm{x}^{adv}_{t},y)].$
    \Until {$t = T$}
    \Ensure  $\bm{x}_{T}^{adv}$.
    \end{algorithmic}
\end{algorithm}

According to the analysis in Section 3.1, we can get a fast gradient sign method from the constrained HB momentum method (\ref{constrained HB}), i.e.,
\begin{equation}\label{HBI}
\bm{x}_{t+1}^{adv}= P_{\mathbf{Q}} [\bm{x}_{t}^{adv} + \alpha sign (\nabla_{\bm{x}} J(\bm{x}_{t}^{adv},y)) +\mu (\bm{x}_{t}^{adv}-\bm{x}_{t-1}^{adv})].
\end{equation}

Naturally, we call this algorithm HBI-FGSM. Similar to I-FGSM, HBI-FGSM only uses the current gradient to adapt the step-size. Based on HBI-FGSM, the detailed steps of our adaptive MI-FGM are described in Algorithm \ref{alg:AdaMI-FGM}.

\begin{algorithm}
\caption{AdaMI-FGM}
\label{alg:AdaMI-FGM}
\begin{algorithmic}[1]
    \Require
    A target classifier $f$ with loss function $J$, a benign image $\bm{x}$ with its ground-truth label $y$.
    \Require
    The perturbation size $\epsilon$, step-size parameter $\alpha>0$, constant $\delta > 0$, EMA parameter $\beta>0$, momentum parameter $\mu>0$ and total number of iterations $T$.
    \State Initialize $\bm{g}_0=0$, $\bm{x}^{adv}_0=\bm{x}$ and $V_0 = \mathbf{0}_{d\times d}$.
    \Repeat
    \State Update $V_t$ by Eq. (\ref{arithmetic average}) or Eq. (\ref{EMA}).
    \State $\hat{V_t} = V_t^{\frac{1}{2}}+\displaystyle\frac{\delta }{\sqrt{t}}I_d$.
    \State $
    \begin{aligned}
    & \bm{x}_{t+1}^{adv}= P_{\mathbf{Q}} [\bm{x}_{t}^{adv} + \displaystyle\frac{\alpha}{\sqrt t} \hat{V_t}^{-1} \nabla_{\bm{x}}J(\bm{x}^{adv}_{t},y) \\
    & \ \ \ \ \ \ \ \ \ \ \ \ \ \ \ \ \ \ \ \ \  +\mu (\bm{x}_{t}^{adv}-\bm{x}_{t-1}^{adv})].
    \end{aligned}
    $
    \Until {$t = T$}\\
    \Return $\bm{x}^{adv}=\bm{x}^{adv}_T$.
\end{algorithmic}
\end{algorithm}

In contrast to AdaI-FGM, there is an additional momentum term $\mu (\bm{x}_{t}^{adv}-\bm{x}_{t-1}^{adv})$ in AdaMI-FGM. Similar to the relationship between PGD (\ref{PGD}) and HB method (\ref{constrained HB}), AdaMI-FGM can be regarded as a natural extension of AdaI-FGM in terms of the HB momentum. It is worth mentioning that our AdaMI-FGM is established upon the constrained HB momentum method (\ref{constrained HB}) rather than MI-FGSM.

In Algorithm \ref{alg:AdaI-FGM} and \ref{alg:AdaMI-FGM}, a vanishing factor $\frac{\delta }{\sqrt{t}}I$ is added to the diagonal of $V_t$ and get $\hat{V_t}$. Such an operation is a standard technique to avoid too large steps caused by zero or small gradients in the beginning iterations \cite{Ruder2016AnOO}.

Note that we don't set $\alpha={\epsilon}/{T}$ like MI-FGSM in \cite{Dong2018BoostingAA}, because this setting is too restricted in limiting $\bm{x}^{adv}_{t}$ ($t=1,2,\ldots,T-1$).

It is easy to find that the main difference between the available gradient-based algorithms and our adaptive ones lies in the way of normalizing the scale of the gradients. The former uses the sign function or $l_1$-norm (or the current gradient information) while the latter employs the step-size with the accumulated gradients.

For convex and strongly convex functions, AdaGrad (\ref{AdaGrad}) enjoys the same order of convergence rates of as PGD (\ref{PGD}) \cite{Duchi2010AdaptiveSM, wang2019sadam}. In general convex cases, it was proved that AdaMI-FGM with a time-varying step size has optimal individual convergence ($O(1/\sqrt{t})$) \cite{Tao2021TheRO}, which is better than that of SGD-like algorithms ($O(log t/\sqrt{t})$). As AdaMI-FGM with a constant learning rate is used in this paper to generate adversarial examples, we give its convergence proof in the Appendix \ref{sec:appendix}  (see Theorem \ref{thm:bigtheorem} and Theorem \ref{thm:bigtheorem1}). Note that the derived convergence rate is better than that of SGD for sparse problems.

Note that $l_1$-norm plays an important in boosting the adversarial attacks for MI-FGM (\ref{MI}). To further illustrate the role of the accumulated gradients in the step-size strategy, we give another adaptive strategy of arithmetic average of $l_1$-norm of historical gradients, i.e.,
\begin{equation}\label{l1-adaptive}
\hat{V}_{t}=\frac{\sum_{j=1}^{t}{{\|\nabla_{\bm{x}} J(\bm{x}^{adv}_{j},y)}\|_1}+\delta} {t} I_d.
\end{equation}

Like the adaptive strategies (\ref{arithmetic average}) and (\ref{EMA}), we can obtain an adaptive algorithm for adversarial attacks. The main difference between this algorithm and MI-FGSM (\ref{MI}) is that the $l_1$-norm of current $\nabla_{\bm{x}} {J(\bm{x}^{adv}_{t},y)}$ in (\ref{MI}) is replaced by the arithmetic average of $l_1$-norm of historical gradients.

\section{Experiments}
\label{sec:Experiments}
In this section, we conduct extensive experiments to demonstrate the performance of the proposed AdaI-FGM and AdaMI-FGM. The success rate of each attack denotes its misclassification rate of the corresponding models with adversarial examples as inputs.

Our purpose here is mainly to show the benefits brought by the integrated adaptive strategies rather than the comprehensive performance evaluation against a wide range of attack methods and models. Typically, we only focus on comparing I-FGSM, HBI-FGSM and MI-FGSM with the adaptive algorithms AdaI-FGM and AdaMI-FGM.

For clarity, we use the denotation AdaI-FGM(1) and AdaMI-FGM(1) when using the arithmetic average of AdaGrad (\ref{arithmetic average}), and AdaI-FGM(2), AdaI-FGM(3) and AdaMI-FGM(2), AdaMI-FGM(3) when using EMA (\ref{EMA}) and arithmetic average of $l_1$-norm (\ref{l1-adaptive}) respectively.


\subsection{Datasets, models and parameters }
\label{sec:4.1}

We will use the same datasets and models as that in \cite{Dong2018BoostingAA, Lin2020NesterovAG}. For convenience, the datasets and models are detailed in the following.

1000 images are randomly selected, which belong to the 1000 categories from ILSVRC 2012 validation set \cite{Russakovsky2015ImageNetLS}. Note that these images are resized to 299$\times$299$\times$3 and they are almost correctly classified by all the testing models.

We use both normally and adversarially trained models. For normally trained models, we choose Inception-v3 (Inc-v3) \cite{Szegedy2016RethinkingTI}, Inception-v4 (Inc-v4), Inception-Resnet-v2 (IncRes-v2) \cite{Szegedy2017Inceptionv4IA} and Resnet-v2-101 (Res-101) \cite{Ilyas2018BlackboxAA}. For adversarially trained models, we select Inc-v3ens3, Inc-v3ens4 and IncRes-v2ens \cite{Tramr2018EnsembleAT}.

Among all the experiments, the parameters of optimization problem (\ref{adv-optimization}) is fixed, i.e., we set the maximum perturbation $\epsilon=16$ and the total number of iteration $T = 10$. Following \cite{Dong2018BoostingAA}, the step-size $\alpha=\epsilon/T$ in I-FGSM and MI-FGSM to make the generated adversarial examples satisfy the $l_\infty$ ball constraints. For MI-FGSM, $\mu=1$ \cite{Dong2018BoostingAA}. As the usual selection in RMSProp and Adam, we set the decay factor $\beta=0.999$ in EMA (\ref{EMA}). For AdaI-FGM and AdaMI-FGM, similar to the experiments in \cite{wang2019sadam} and \cite{Tao2021TheRO}, the vanishing factor $\delta$ is usually set to be a small number. In this paper, we choose $\delta=10^{-6}$.

\subsection{Comparison between AdaI-FGM and I-FGSM}
\label{sec:4.2}


Note that the concerned methods only have one important adjustable parameter, i.e., the step-size parameter $\alpha$. We choose $\alpha$ from the set of $\{2^{-7}, 2^{-6},2^{-5},2^{-4},2^{-3}, 2^{-2},2^{-1}\}$ for all experiments. In AdaI-FGM(1), we choose $\alpha=2^{-4}$. In AdaI-FGM(2), we select $\alpha=2^{-7}$. For AdaI-FGM(3), we set $\alpha=2^{-5}$.

\begin{figure*}[tbh]\small
\setlength{\abovecaptionskip}{0.1cm}
\setlength{\belowcaptionskip}{-0.5cm}
	\centering
	\begin{minipage}{0.33\linewidth}
		\centering
		\includegraphics[width=1\textwidth,height=130pt]{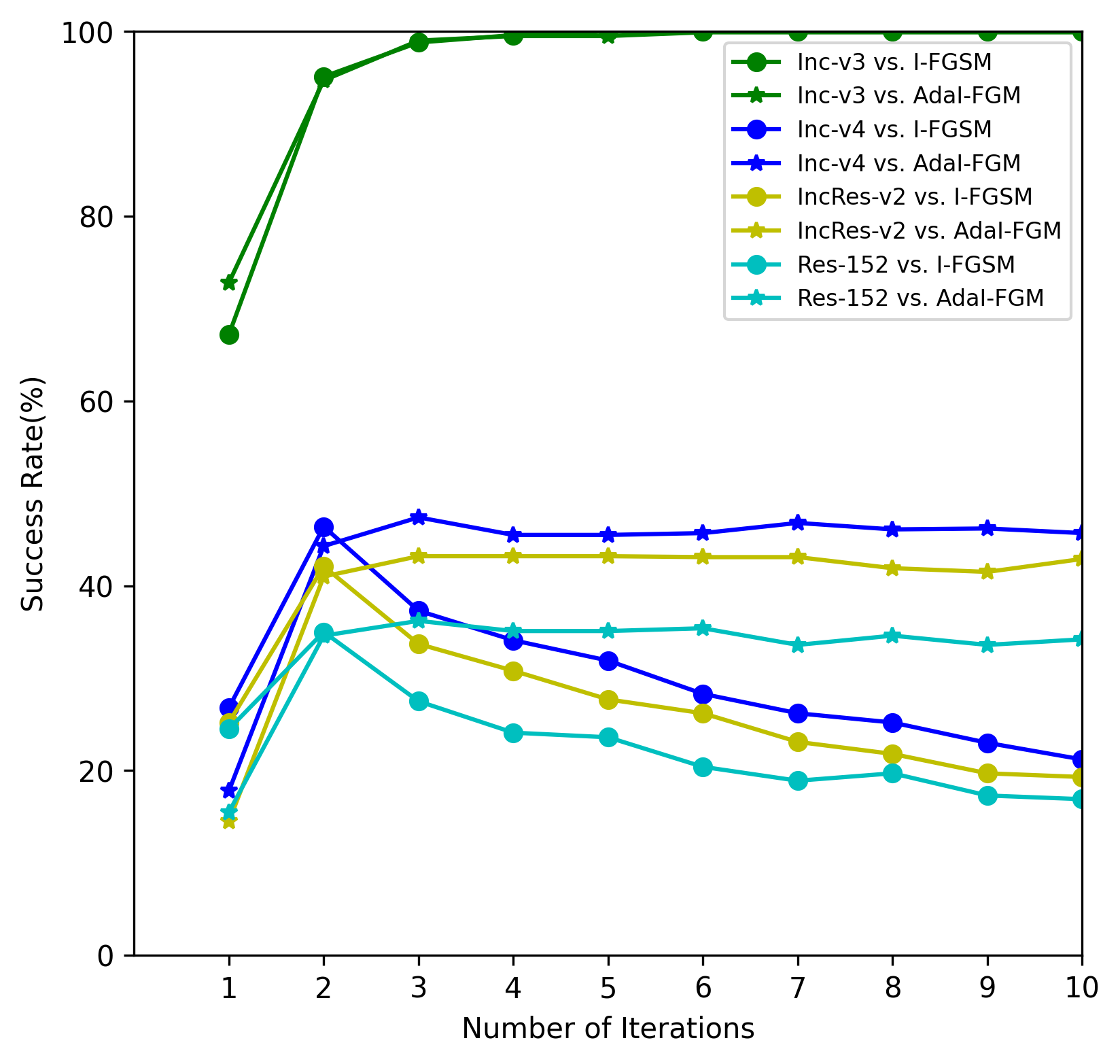}
        \centerline{(a) AdaI-FGM(1)}
	\end{minipage}
	\begin{minipage}{0.33\linewidth}
		\centering
		\includegraphics[width=1\textwidth,height=130pt]{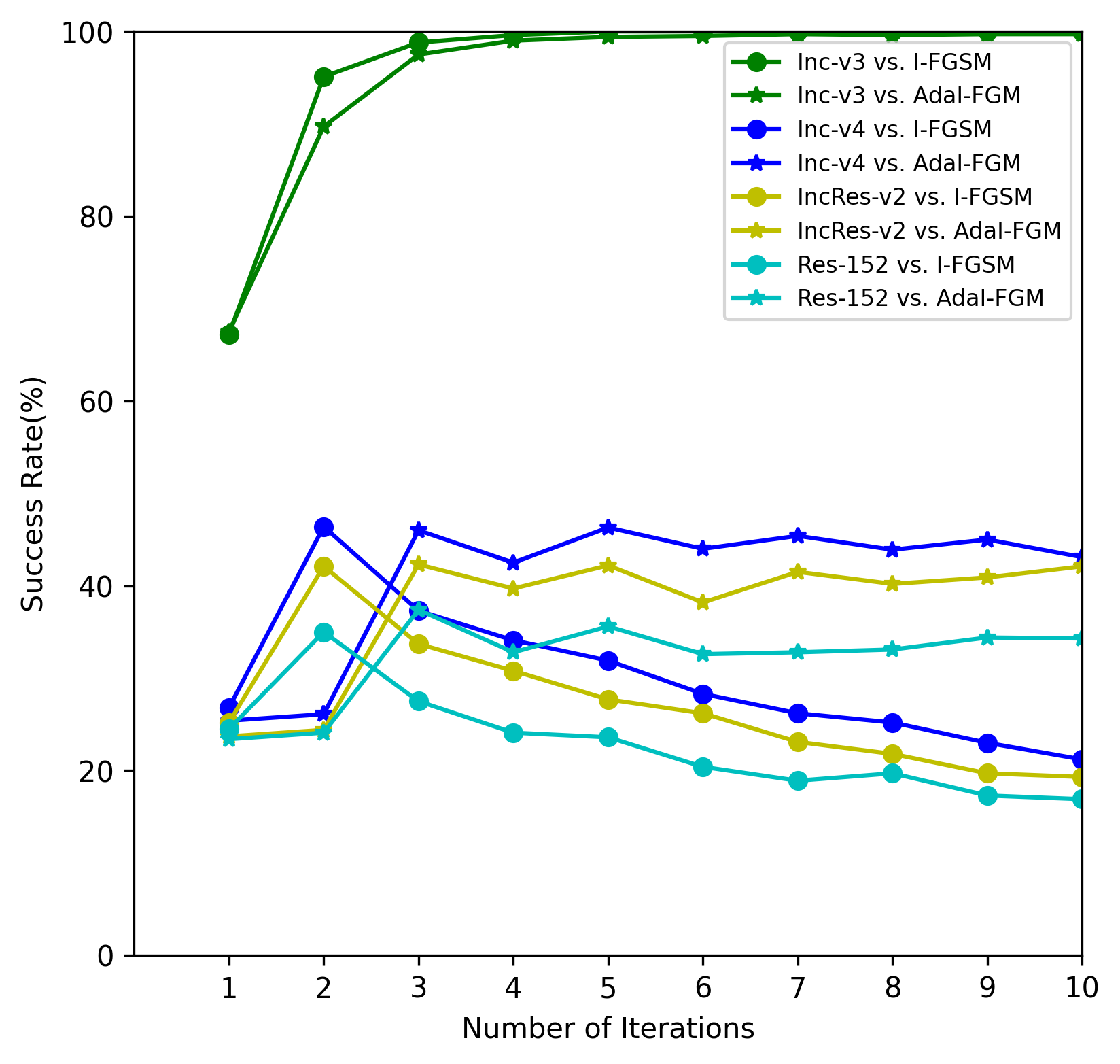}
        \centerline{(b) AdaI-FGM(2)}
	\end{minipage}
    \begin{minipage}{0.33\linewidth}
		\centering
		\includegraphics[width=1\textwidth,height=130pt]{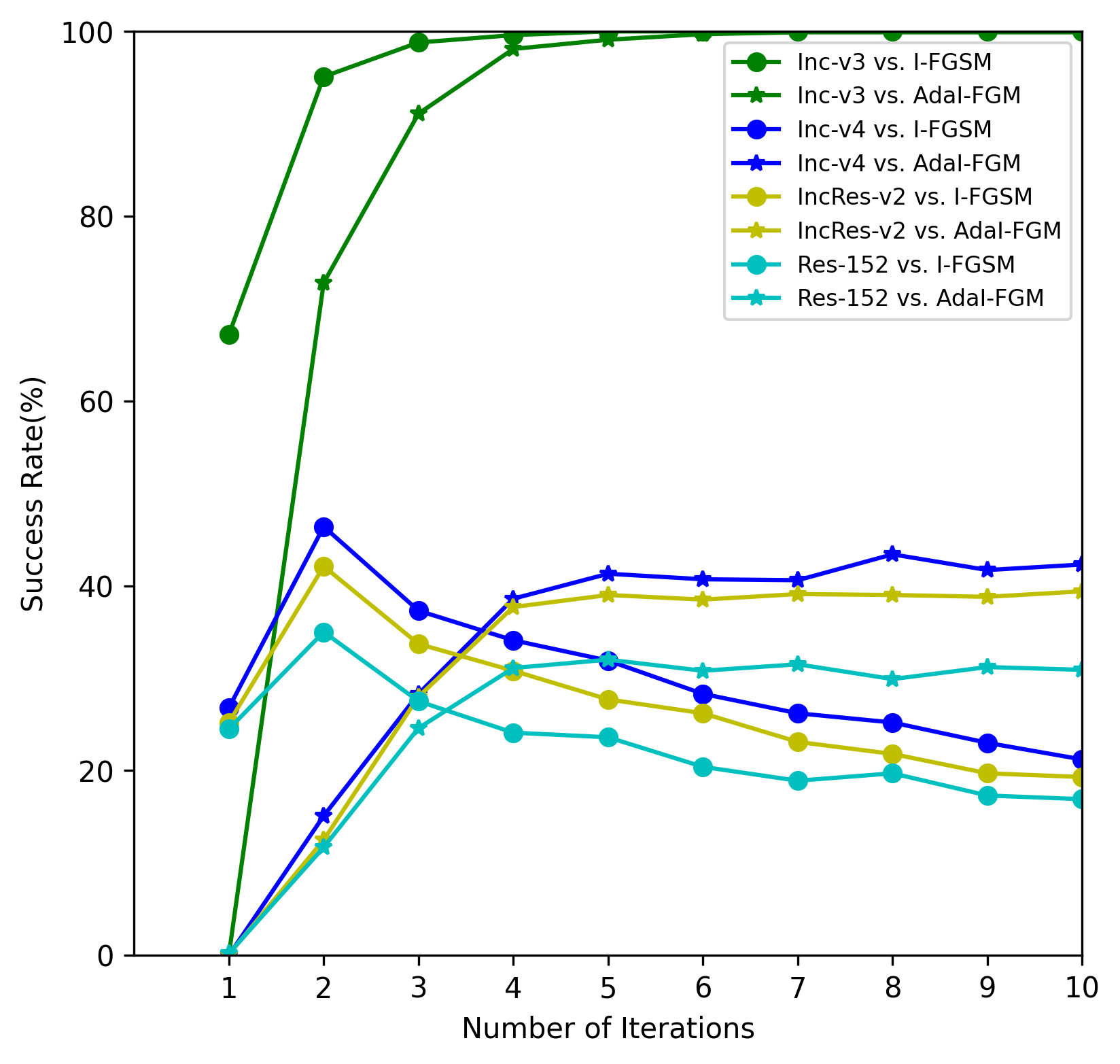}
        \centerline{(c) AdaI-FGM(3)}
	\end{minipage}
\caption{\normalsize{Attack success rates (\%) of the adversarial examples generated for Inc-v3 model against Inc-v3 (white-box), Inc-v4, IncResv2 and Res-152 (black-box). We compare the results of I-FGSM and AdaI-FGM with different iterations. }}
\label{fig:AI}
\end{figure*}

\begin{table*}[t]\small
\setlength{\abovecaptionskip}{0.1cm}
\setlength{\belowcaptionskip}{-0.2cm}
\captionsetup{font={normalsize}}
\caption{\normalsize{Attack success rates (\%) of adversarial attacks against baseline models. The adversarial examples are crafted on Inc-v3, Inc-v4, IncRes-v2, and Res-101 respectively using I-FGSM and AdaI-FGM. $^*$ indicates the white-box attacks.}}
\label{I_and_AI}
\centering
\begin{tabular}{c|c|ccccccc}
	\toprule 
	Model & Attack  &  Inc-v3 & Inc-v4 & IncRes-v2 & Res-101 & Inc-v3ens3 & Inc-v3ens4 & IncRes-v2ens \\
	\midrule 
	\multirow{4}{*}{Inc-v3}& I-FGSM & 100.0$^*$ & 21.2 & 19.3 & 16.9 & 5.3 & 5.0 & 2.8\\ 
	\multirow{4}{*}{ }& AdaI-FGM(1) & 99.9$^*$ & 45.7 & 42.9 & 34.2 & 9.2 & 9.6 &  5.3  \\
    \multirow{4}{*}{ }& AdaI-FGM(2) & 99.7$^*$ & 43.1 & 42.1 & 34.3 & 9.8 & 10.3 &  4.9   \\
    \multirow{4}{*}{ }& AdaI-FGM(3) & 99.9$^*$ & 42.3 & 39.4 & 30.9 & 10.8 & 11.8 & 5.2     \\ \hline
    \multirow{4}{*}{Inc-v4}& I-FGSM & 31.3  & 99.8$^*$& 21.0 & 20.1 & 5.9 & 6.7 & 3.4  \\ 
	\multirow{4}{*}{ }& AdaI-FGM(1) &  57.2 & 100$^*$& 44.1 & 39.5 & 12.2 & 11.0 &  5.9  \\
    \multirow{4}{*}{ }& AdaI-FGM(2) &  58.7 &  99.8$^*$& 43.7 & 38.3 & 11.4 & 11.9 &  6.3   \\
    \multirow{4}{*}{ }& AdaI-FGM(3) &  54.3 &  99.9$^*$& 41.1 & 36.6 & 14.7 & 12.4 &  6.5     \\ \hline
    \multirow{4}{*}{IncRes-v2}& I-FGSM & 32.2  & 25.2 & 97.9$^*$& 21.0 & 6.7 & 7.0 & 4.1 \\ 
	\multirow{4}{*}{ }& AdaI-FGM(1) & 56.6 & 47.5 & 99.4$^*$ & 38.3 &15.2  & 11.5 & 9.0 \\
    \multirow{4}{*}{ }& AdaI-FGM(2) & 62.6 & 49.5 & 99.6$^*$ & 40.0 & 14.1 & 11.8 &  7.9   \\
    \multirow{4}{*}{ }& AdaI-FGM(3) & 57.0 & 46.5 & 98.7$^*$ & 37.8 & 17.9 & 12.3 &  100  \\ \hline
    \multirow{4}{*}{Res-101}& I-FGSM & 30.2  & 26.5 & 22.8 & 99.2$^*$ & 7.1 & 8.1 & 5.4 \\ 
	\multirow{4}{*}{ }& AdaI-FGM(1) & 57.8 & 53.6 & 49.8 & 99.0$^*$ & 16.7 & 16.1 &  8.6  \\
    \multirow{4}{*}{ }& AdaI-FGM(2) & 58.8 & 53.2 & 49.2 & 98.5$^*$ & 14.4 & 13.6 &   8.2    \\
    \multirow{4}{*}{ }& AdaI-FGM(3) & 55.8 & 48.4 & 45.2 & 99.0$^*$ & 21.4 & 17.1 & 10.1      \\ 
	\bottomrule 
	\end{tabular}
\end{table*}

The success rates of attacks against normally and adversarially trained models are reported in Table \ref{I_and_AI}.
It can be observed that the proposed adaptive methods keep a strong white-box adversary like I-FGSM since they can attack a white-box model with a near 100$\%$ success rate. For black-box attacks, it can be seen that our adaptive methods consistently outperform I-FGSM. In fact, each of our AdaI-FGSM achieves nearly twice success rates of I-FGSM.

Note that each AdaI-FGM only slightly modifies I-FGSM in using a different average strategy about the past gradients to replace the current gradient as its step-size. However, the experimental results shows that this simple operation can remarkably boosts the adversarial attacks, verifying the effectiveness of our AdaI-FGM methods.

It is worth indicating that the success rates of our AdaI-FGM is even approaching the level of the state-of-the-art MI-FGSM (see Table \ref{MI_and_AMI1}). Such a fact clearly illustrates our motivation in this paper, i.e., adapting the step-size plays almost the same part as updating the iterative direction in boosting the success rates when the accumulated gradients are integrated.

To illustrate the stability of AdaI-FGM methods, we investigate the changing behaviour of success rates with respect to the number of iteration. For convenience, we only consider the arithmetic average of AdaGrad (\ref{arithmetic average}) and the adversarial examples are crafted on Inc-v3. 
The success rates of adversarial examples are evaluated against Inc-v3, Inc-v4, IncRes-v2 and Res-101 models. The relationships between the success rates and the number of iterations are shown in Fig.\ref{fig:AI}.

It can be observed that when the number of iterations increases, the success rates of I-FGSM for black-box attacks decrease while that of AdaI-FGM maintains at a relatively stable value. These experimental results illustrate that I-FGSM can easily overfit the white-box attacks but suffering from poor transferability. Such a phenomena has already been pointed out in \cite{Dong2018BoostingAA}. Fortunately, as can be seen from Fig.\ref{fig:AI}, our AdaI-FGM methods effectively alleviate the trade-off between the white-box attacks and the transferability. So, we can say that AdaI-FGM significantly improves the performance of I-FGSM.

\subsection{Comparison between HBI-FGSM, AdaMI-FGM and MI-FGSM }

For simplicity, in this subsection, we only focus on the AdaMI-FGM using the arithmetic average of AdaGrad (\ref{arithmetic average}).  We use the same step-size $\alpha$ as that in Section 4.1. Now, there is only one important adjustable parameter in HBI-FGSM and AdaMI-FGM, i.e., the momentum parameter $\mu$. We choose $\mu=0.07$, which is obviously different from $\mu=1$ in MI-FGM. Note that when $\mu=0$, AdaMI-FGM will become AdaI-FGM and HBI-FGSM will become I-FGSM. As AdaI-FGM performs well (see Table \ref{I_and_AI}), we only want to slightly modify the effect of AdaI-FGM by setting a small momentum parameter $\mu$.

\begin{table*}[t]\small
\setlength{\abovecaptionskip}{0.1cm} 
\setlength{\belowcaptionskip}{-0.4cm}
\captionsetup{font={normalsize}}
\caption{\normalsize{Attack success rates (\%) of adversarial attacks against normally trained and adversarially trained models. The adversarial examples are crafted on Inc-v3, Inc-v4, IncRes-v2, and Res-101 respectively by HBI-FGSM, AdaMI-FGM and MI-FGSM. $^*$ indicates the white-box attacks.}}
\label{MI_and_AMI1}
\centering
\begin{tabular}{c|c|ccccccc}
 \toprule 
 Model & Attack  &  Inc-v3 & Inc-v4 & IncRes-v2 & Res-101 & Inc-v3ens3 & Inc-v3ens4 & IncRes-v2ens \\
 \midrule 
 \multirow{3}{*}{Inc-v3}& HBI-FGSM & 99.8$^*$ & 32.3 & 28.2 & 22.5 & 6.6 & 7.7 & 4.0 \\
    \multirow{3}{*}{ } & AdaMI-FGM & 99.9$^*$ & 46.2 & 42.2 & 34.4 & 8.8 & 9.2 & 5.1 \\
    \multirow{3}{*}{ } & MI-FGSM & 100.0$^*$ & 44.0 & 40.8 & 35.1 & 13.7 & 13.2 & 6.1\\ \hline
    \multirow{3}{*}{Inc-v4}& HBI-FGSM & 41.4 & 99.7$^*$ & 28.9 & 25.8 & 7.4 & 7.6 & 4.8  \\
    \multirow{3}{*}{ }  & AdaMI-FGM & 56.9  & 100.0$^*$& 45.2 & 38.7 & 11.8 & 10.5 & 6.1\\
    \multirow{3}{*}{ }& MI-FGSM & 55.7  & 99.8$^*$& 47.4 & 42.0 & 16.2 & 14.9 & 7.4\\ \hline
    \multirow{3}{*}{IncRes-v2}& HBI-FGSM & 43.8 & 32.0 & 98.7$^*$  & 24.4  & 7.0 & 7.6 & 4.7 \\
    \multirow{3}{*}{ } & AdaMI-FGM & 57.6 & 47.6 & 99.5$^*$ & 38.2 & 14.3 & 11.2 & 8.2\\
    \multirow{3}{*}{ }& MI-FGSM & 59.3 & 50.5 & 97.8$^*$& 46.1  & 23.2 & 17.4 & 11.4 \\ \hline
    \multirow{3}{*}{Res-101}& HBI-FGSM & 44.6  & 38.9 & 34.8 & 99.3$^*$ &  9.0 & 9.1 & 5.5 \\
    \multirow{3}{*}{ } & AdaMI-FGM & 60.3  & 54.3 & 51.3 & 99.0$^*$ & 17.7 &  14.7& 8.0\\
    \multirow{3}{*}{ } & MI-FGSM & 57.9 & 49.2 & 49.2 & 99.3$^*$ & 24.5 & 21.6 & 12.5  \\
 \bottomrule 
 \end{tabular}
\end{table*}

The success rates of attacks against normally and adversarially trained models are reported in Table \ref{MI_and_AMI1}.
It is easy to find that AdaMI-FGSM consistently outperforms HBI-FGSM. This experimental result further illustrates the role of adapting step-size with the accumulated gradients in improving the transferability of adversarial attacks for gradient-based algorithms.

What is more, we can also see that AdaMI-FGM outperforms MI-FGSM 8 times in total 16 attacks against normally trained models, i.e., AdaMI-FGM can reach almost the same success rates as the state-of-the-art MI-FGSM. So, AdaMI-FGM remains competitive with MI-FGSM for black-box attacks. Unfortunately, from Table \ref{MI_and_AMI1}, AdaMI-FGM is consistently inferior to MI-FGSM in attacks against adversarially trained models. This phenomena can be interpreted from the perspective of optimization. Recall that the goal of adversarial attacks is to seek an example that misleads the model decision and this can be cast as optimization problem (\ref{adv-optimization}), i.e., optimization problem (\ref{adv-optimization}) is only good for adversarial attacks. AdaMI-FGM only focuses on solving the induced problem (\ref{adv-optimization}) with specifically treating the adversarially trained models. Instead, one should solve the {\it min-max} problem (\ref{adav-training}) to get a high success rate against adversarially trained models \cite{Madry2018TowardsDL, Wang2019OnTC}.


Nevertheless, AdaMI-FGM is theoretically motivated and it indeed has some practical advantages especially in keeping stability and accelerating convergence for solving the optimization problems. To illustrate the stability of AdaMI-FGM, like that in AdaI-FGM, we show the relationship between the success rate and the number of iterations in Fig.\ref{fig:AMI}.

\begin{figure}[ht]\small
\setlength{\abovecaptionskip}{0.1cm} 
\setlength{\belowcaptionskip}{-0.4cm}
\centering
\includegraphics[width=0.45\textwidth,height=140pt]{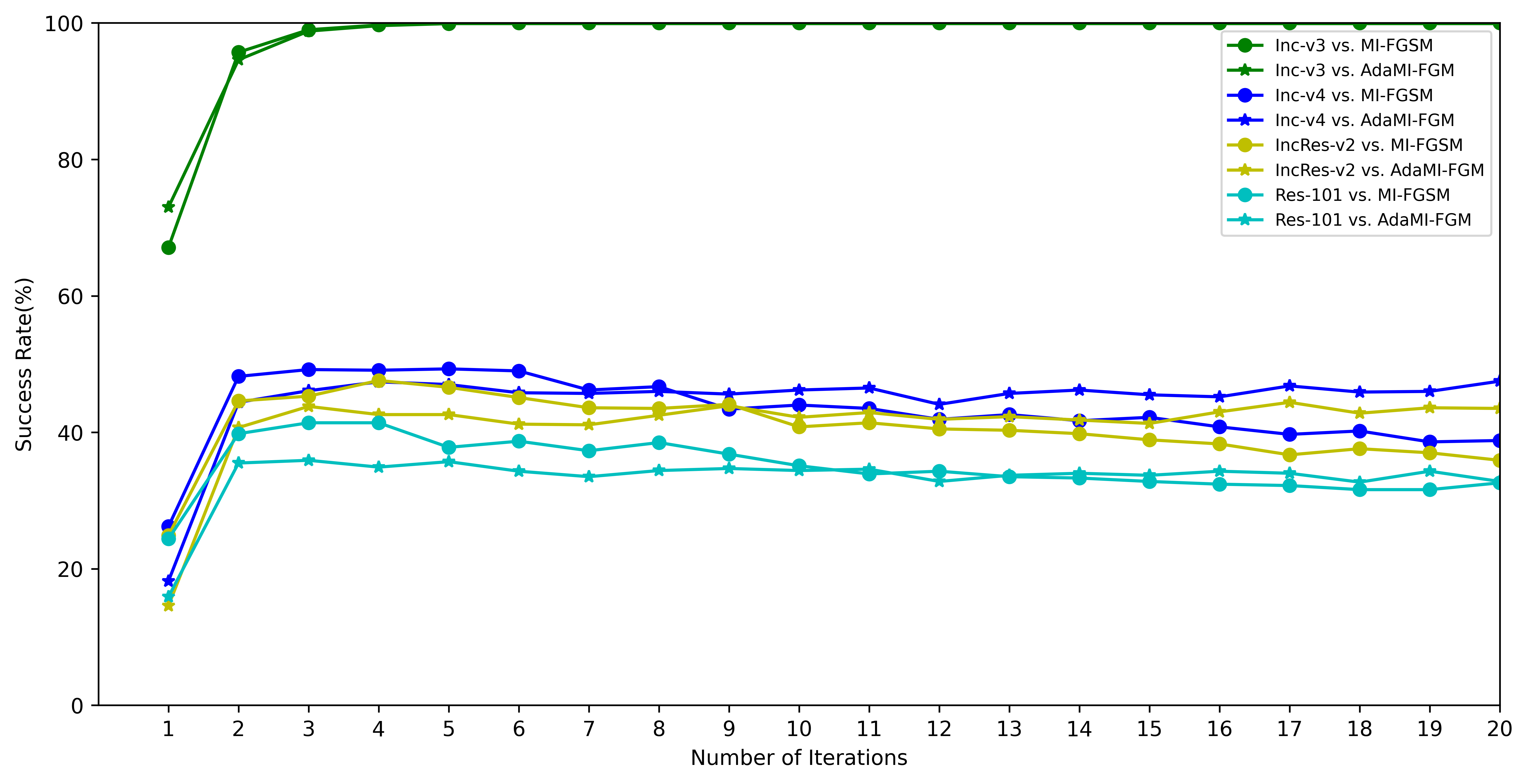}
\caption{\normalsize{Attack success rates (\%) of the adversarial examples generated for Inc-v3 model against Inc-v3 (white-box), Inc-v4, IncResv2 and Res-152 (black-box). We compare the results of MI-FGSM and AdaMI-FGM with different iterations.}}
\label{fig:AMI}
\end{figure}

From Fig.\ref{fig:AI} and Fig.\ref{fig:AMI}, it can be observed that when the number of iterations increases, the success rate of MI-FGSM for black-box attacks looks much more stable than that of I-FGSM. This fact has already been indicated in \cite{Dong2018BoostingAA}. Besides, the phenomena that using momentum at each iteration can stabilizes the update direction has also been illustrated in \cite{Dong2018BoostingAA}. Fortunately, as shown in Fig.\ref{fig:AI} and Fig.\ref{fig:AMI}, our AdaMI-FGM has the most stable behavior among I-FGSM, MI-FGSM and AdaI-FGM. This illustrates that integrated the accumulate gradients in both step-size and update direction can further stabilizes the whole optimization process.

To make a through comparison between AdaMI-FGM and MI-FGSM, we also investigate the convergence behaviour of loss function $J(\bm{x}_{t}^{adv},y)$ with respect to the number of iterations. The relationship between the value of loss function $J(\bm{x}_{t}^{adv},y)$ and the number of iterations is shown in Fig.\ref{fig:loss}. As can be seen in Fig.\ref{fig:loss}, AdaMI-FGM converges consistently fast than MI-FGSM. In viewpoint of pure optimization algorithms, AdaMI-FGM is more suitable for solving optimization problem (\ref{adv-optimization}) than MI-FGSM.

\begin{figure}[ht]\small
\setlength{\abovecaptionskip}{0.1cm} 
\setlength{\belowcaptionskip}{-0.4cm}
\centering
\includegraphics[width=0.45\textwidth,height=140pt]{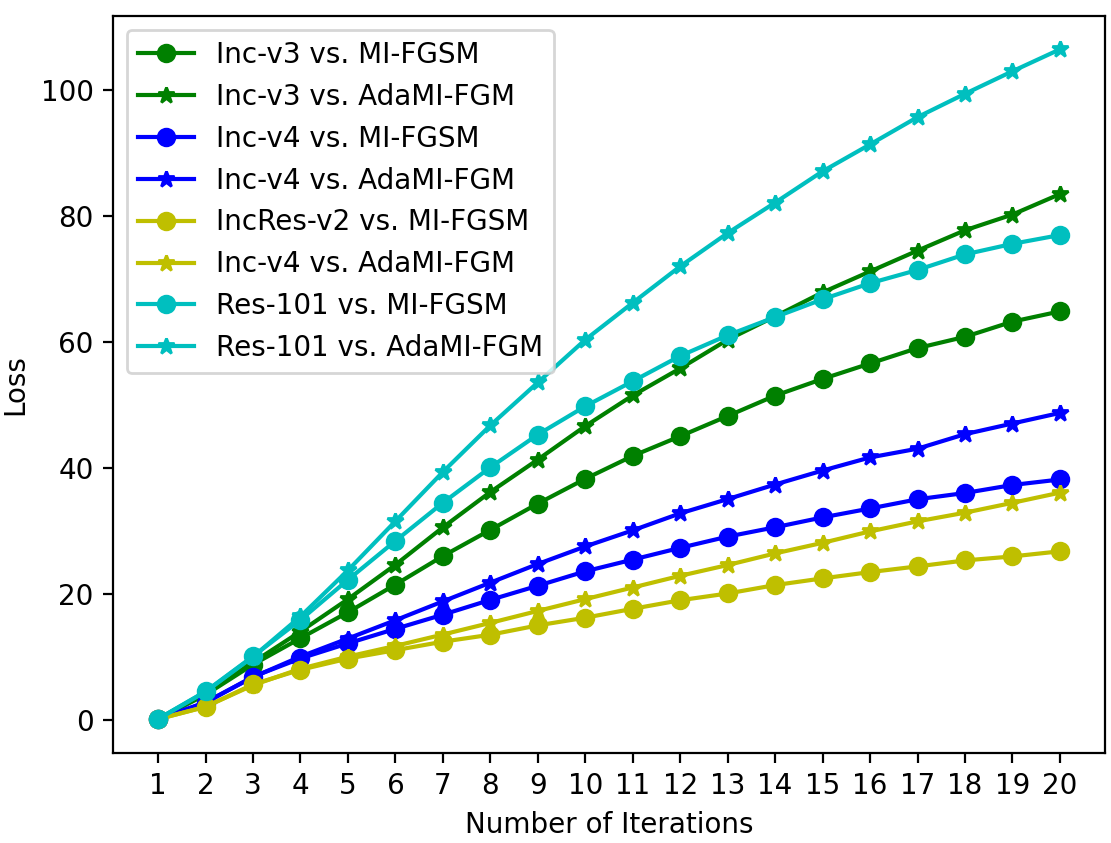}
\caption{\normalsize{Values of loss function of the adversarial examples generated for Inc-v3, Inc-v4, IncResv2 and Res-152 models. We compare the results of MI-FGSM and AdaMI-FGM with different iterations.}}
\label{fig:loss}
\end{figure}

\subsection{Discussion}

It is proved that AdaMI-FGM is an optimal algorithm for solving constrained convex problems \cite{Tao2021TheRO}. Unfortunately, when AdaMI-FGM is applied to deal with adversarial attack tasks, its success rate against MI-FGSM is still unsatisfactory. We think that this deficiency is caused by the loss function $J(\bm{x}^{adv},y)$ in optimization problem (\ref{adv-optimization}). From Fig.\ref{fig:AMI} and Fig.\ref{fig:loss}, it is clearly observed that solving optimization problem (\ref{adv-optimization}) with high accuracy does not imply that the transferability of adversarial examples is boosted. Recent investigation \cite{Zhang2022ProvingCM} has revealed that although many gradient-based methods have been proposed to enhance the transferability of adversarial perturbations, these methods are designed in a heuristic manner. Instead, they have proposed a new interaction-based loss function to replace the commonly-used $J(\bm{x}^{adv},y)$ in optimization problem (\ref{adv-optimization}). Both theoretical analysis and experiments show that their new loss significantly improves the adversarial transferability. If we use AdaMI-FGM to solve the induced optimization problem in \cite{Zhang2022ProvingCM}, satisfactory transferability against MI-FGSM may be expected.

\section{Conclusion and future work}

In this paper, we first explain the gradient-based attack methods from the perspective of adapting step-size. Then, we propose a broad class of adaptive gradient-based iterative methods, in which the accumulated gradients of the loss function is further used to adapt the step-size. We answer several theoretical concerns when using gradient-based methods to solve the induced problems. The idea in this paper inspires us to introduce more adaptive and regular optimization methods to the field of adversarial attacks. NAG in its adaptive setting and adaptive methods for solving interaction-based optimization problem in \cite{Zhang2022ProvingCM} and {\it min-max} problem (\ref{adav-training}) will be investigated in the future work.




\bibliographystyle{unsrt}
\bibliography{egbib}

\appendix
\clearpage
\onecolumn

\section{Appendix}
\subsection{Convergence Analysis of AdaMI-FGM with a Constant Learning Rate}
\label{sec:appendix}

To analyze convergence of AdaMI-FGM, we require some assumptions.
To analyze convergence of AdaMI-FGM, we require some assumptions.

\noindent {\bf Assumption 1.} {\it Assume that the objective function $J(\boldsymbol{x}, y)$ is concave and there exist constants $G > 0$ and $G_{\infty} > 0$ such that
\begin{equation*}
  \| \nabla_{\boldsymbol{x}}J(\boldsymbol{x}^{adv},y)\| \leq G \ \text{and} \ \| \nabla_{\boldsymbol{x}}J(\boldsymbol{x}^{adv},y)\|_{\infty} \leq G_{\infty}, \ \  \forall \boldsymbol{x}^{adv} \in \mathbf{Q},
\end{equation*}
\label{ass:xfinite}}
where $\| \cdot \|$ represents the $L_2$-norm.

Generating a non-targeted adversarial example $\boldsymbol{x}^{adv}$ from
a real example $\boldsymbol{x}$ can be formulated as the following optimization problem
\begin{equation}\label{adv-optimization}
\max J(\boldsymbol{x}^{adv},y), \ s. t. \  \mathbf{Q}=\|\boldsymbol{x}^{adv}-\boldsymbol{x}\|_p \leq \epsilon.
\end{equation}

AdaMI-FGM for solving problem (\ref{adv-optimization}) is
\begin{equation}\label{algorithm2}
   \begin{aligned}
   \boldsymbol{x}^{adv}_{t+1}=P_\mathbf{Q}[\boldsymbol{x}^{adv}_t+\alpha \hat{V}^{-1}_t\boldsymbol{g}_t+\mu(\boldsymbol{x}^{adv}_t-\boldsymbol{x}^{adv}_{t-1})],
   \end{aligned}
 \end{equation}
where
\begin{equation*}
  \begin{aligned}
     \boldsymbol{g}_t=\nabla_{\boldsymbol{x}}J(\boldsymbol{x}^{adv}_{t}, y), \ \hat{V}_t = V_t^{\frac{1}{2}} + \frac{\delta}{\sqrt{t}}I_d \ \text{and} \  V_t = \beta_t V_{t-1} + (1-\beta_t)diag(\boldsymbol{g}_t\boldsymbol{g}_t^T).
  \end{aligned}
\end{equation*}

In order to make our proof easy to understand, we first consider the MI-FGM (\ref{algorithm1}), which is a specific AdaMI-FGM without using the adaptive step size strategy. For simplicity, $J(\boldsymbol{x}^{adv}_{t}, y)$ will be rewritten as $J(\boldsymbol{x}^{adv}_{t})$ in the following.

Naturally, MI-FGM can be formulated as
 \begin{equation}\label{algorithm1}
   \begin{aligned}
   \boldsymbol{x}^{adv}_{t+1}=P_\mathbf{Q}[\boldsymbol{x}^{adv}_t+\alpha\boldsymbol{g}_t+\mu(\boldsymbol{x}^{adv}_t-\boldsymbol{x}^{adv}_{t-1})].
   \end{aligned}
 \end{equation}

\begin{lemma} \label{lem:usefullemma}
(Dimitri P. et al., 2003) For $\boldsymbol{x} \in \mathbb{R}^d$, $\boldsymbol{x}_0 \in \mathbf{Q}$,
 \begin{equation*}
   \begin{aligned}
     \langle \boldsymbol{x}-\boldsymbol{x}_{0}, \boldsymbol{u}-\boldsymbol{x}_0\rangle \leq 0
   \end{aligned}
 \end{equation*}
for all $\boldsymbol{u} \in \mathbf{Q}$ if and only if $\boldsymbol{x}_0=P(\boldsymbol{x})$.
\end{lemma}

\begin{theorem}
\label{thm:bigtheorem}
Let $\boldsymbol{x}^{\ast} $ be a solution of problem (\ref{adv-optimization}). Assume that $\mathbf{Q}$ is bounded and $0 \leq \mu < 1$. Let $\{ \boldsymbol{x}^{adv}_t\}_{t=1}^{\infty} $ be generated by MI-FGM (\ref{algorithm1}).
Set\\
\begin{gather*}
  \mathbf{p}_t=\frac{\mu}{1-\mu}(\boldsymbol{x}^{adv}_t-\boldsymbol{x}^{adv}_{t-1}) \ \text{and} \ \alpha \equiv \frac{C}{\sqrt{T}},
\end{gather*}
where $T$ is the total number of iterations.
Then we have
$$  \boldsymbol{x}^{adv}_{t+1}+\mathbf{p}_{t+1}=P_{\mathbf{Q}}[\boldsymbol{x}^{adv}_t+\mathbf{p}_{t}+\frac{\alpha}{1-\mu}\boldsymbol{g}_t]$$
  and
  $$ J(\boldsymbol{x}^\ast)-J(\boldsymbol{\bar{x}}^{adv}_T) \leq O(\frac{1}{\sqrt{T}}),$$
where $\boldsymbol{\bar{x}}^{adv}_T = \frac{1}{T}\sum_{t=1}^{T}\boldsymbol{x}^{adv}_t$.
\end{theorem}
\begin{proof}
 According to the iteration (\ref{algorithm1}) and Lemma \ref {lem:usefullemma}, $\forall \boldsymbol{x}^{adv} \in \mathbf{Q}$, we have
 \begin{gather*}
   \langle \boldsymbol{x}^{adv}_{t+1}-\boldsymbol{x}^{adv}_{t}-\mu(\boldsymbol{x}^{adv}_{t}-\boldsymbol{x}^{adv}_{t-1})-\alpha\boldsymbol{g}_t, \boldsymbol{x}^{adv}_{t+1}-\boldsymbol{x}^{adv}\rangle \leq 0.
 \end{gather*}
 This is
  \begin{gather*}
   \langle \frac{1}{1-\mu}(\boldsymbol{x}^{adv}_{t+1}-\boldsymbol{x}^{adv}_{t})-\mathbf{p}_{t}-\frac{\alpha}{1-\mu}\boldsymbol{g}_t, \boldsymbol{x}^{adv}_{t+1}-\boldsymbol{x}^{adv}\rangle \leq 0,
 \end{gather*}
 i.e.,
  \begin{equation}\label{a}
   \langle \boldsymbol{x}^{adv}_{t+1}+\mathbf{p}_{t+1}-(\boldsymbol{x}^{adv}_{t}+\mathbf{p}_{t})-\frac{\alpha}{1-\mu}\boldsymbol{g}_t, \boldsymbol{x}^{adv}_{t+1}-\boldsymbol{x}^{adv}\rangle \leq 0.
 \end{equation}
 Specifically when $\boldsymbol{x}^{adv}=\boldsymbol{x}^{adv}_t$,
 \begin{equation}\label{b}
   \langle \boldsymbol{x}^{adv}_{t+1}+\mathbf{p}_{t+1}-(\boldsymbol{x}^{adv}_{t}+\mathbf{p}_{t})-\frac{\alpha}{1-\mu}\boldsymbol{g}_t, \frac{\mu(\boldsymbol{x}^{adv}_{t+1}-\boldsymbol{x}^{adv}_t)}{1-\mu}\rangle \leq 0.
 \end{equation}
 From (\ref{a}) and (\ref{b}), we have,
  \begin{equation*}
   \langle \boldsymbol{x}^{adv}_{t+1}+\mathbf{p}_{t+1}-(\boldsymbol{x}^{adv}_{t}+\mathbf{p}_{t})-\frac{\alpha}{1-\mu}\boldsymbol{g}_t, \boldsymbol{x}^{adv}_{t+1}+\mathbf{p}_{t+1}-\boldsymbol{x}^{adv}\rangle \leq 0.
 \end{equation*}
 Using Lemma \ref{lem:usefullemma}, we have
  \begin{equation*}
   \boldsymbol{x}^{adv}_{t+1}+\mathbf{p}_{t+1}=P_{\mathbf{Q}}[\boldsymbol{x}^{adv}_t+\mathbf{p}_{t}+\frac{\alpha}{1-\mu}\boldsymbol{g}_t].
 \end{equation*}
 Then
 \begin{equation*}
   \begin{aligned}
   \|\boldsymbol{x}^{\ast}-(\boldsymbol{x}^{adv}_{t+1}+\mathbf{p}_{t+1}) \|^2 &\leq \|\boldsymbol{x}^{\ast}-(\boldsymbol{x}^{adv}_{t}+\mathbf{p}_{t}) -\frac{\alpha}{1-\mu}\boldsymbol{g}_t\|^2 \\
   & = \|\boldsymbol{x}^{\ast}-(\boldsymbol{x}^{adv}_{t}+\mathbf{p}_{t})\|^2 + \| \frac{\alpha}{1-\mu}\boldsymbol{g}_t\|^2 - 2\langle \frac{\alpha}{1-\mu}\boldsymbol{g}_t , \boldsymbol{x}^{\ast}-\boldsymbol{x}^{adv}_{t}\rangle \\
   & - 2\langle \frac{\alpha \mu}{(1-\mu)^2}\boldsymbol{g}_t , \boldsymbol{x}^{adv}_{t-1}-\boldsymbol{x}^{adv}_{t}\rangle.\\
   \end{aligned}
 \end{equation*}
 According to the properties of concave function,
 \begin{equation*}
   \langle \boldsymbol{g}_t , \boldsymbol{x}^{\ast}-\boldsymbol{x}^{adv}_{t}\rangle \geq J(\boldsymbol{x}^{\ast})-J(\boldsymbol{x}^{adv}_{t}) \ and
   \ \langle \boldsymbol{g}_t,  \boldsymbol{x}^{adv}_{t-1}-\boldsymbol{x}^{adv}_{t}\rangle \geq J(\boldsymbol{x}^{adv}_{t-1})-J(\boldsymbol{x}^{adv}_{t}).
 \end{equation*}
 Then
 \begin{equation*}
   \begin{aligned}
   \|\boldsymbol{x}^{\ast}-(\boldsymbol{x}^{adv}_{t+1}+\mathbf{p}_{t+1}) \|^2 & \leq \|\boldsymbol{x}^{\ast}-(\boldsymbol{x}^{adv}_{t}+\mathbf{p}_{t})\|^2 + \frac{\alpha^2}{(1-\mu)^2}\|\boldsymbol{g}_t\|^2 \\
   & + \frac{2\alpha}{1-\mu}[J(\boldsymbol{x}^{adv}_{t})-J(\boldsymbol{x}^{\ast})] + \frac{2\alpha \mu}{(1-\mu)^2}[J(\boldsymbol{x}^{adv}_{t})-J(\boldsymbol{x}^{adv}_{t-1})].\\
   \end{aligned}
 \end{equation*}
Rearrange the inequality, we have
 \begin{equation*}
   \begin{aligned}
    \frac{2\alpha}{1-\mu}[J(\boldsymbol{x}^{\ast})-J(\boldsymbol{x}^{adv}_{t})] & \leq \frac{2\alpha \mu}{(1-\mu)^2}[J(\boldsymbol{x}^{adv}_{t})-J(\boldsymbol{x}^{adv}_{t-1})] +\| \boldsymbol{x}^{\ast}-(\boldsymbol{x}^{adv}_{t}+\mathbf{p}_{t})\|^2\\
   & - \|\boldsymbol{x}^{\ast}-(\boldsymbol{x}^{adv}_{t+1}+\mathbf{p}_{t+1}) \|^2 + \frac{\alpha^2}{(1-\mu)^2}\|\boldsymbol{g}_t\|^2,\\
   \end{aligned}
 \end{equation*}
 i.e.,
 \begin{equation*}
   \begin{aligned}
    J(\boldsymbol{x}^{\ast})-J(\boldsymbol{x}^{adv}_{t}) & \leq \frac{\mu}{1-\mu}[J(\boldsymbol{x}^{adv}_{t})-J(\boldsymbol{x}^{adv}_{t-1})] + \frac{1-\mu}{2\alpha}[\|\boldsymbol{x}^{\ast}-(\boldsymbol{x}^{adv}_{t}+\mathbf{p}_{t})\|^2\\
   & - \|\boldsymbol{x}^{\ast}-(\boldsymbol{x}^{adv}_{t+1}+\mathbf{p}_{t+1}) \|^2] + \frac{\alpha}{2(1-\mu)}\|\boldsymbol{g}_t\|^2.\\
   \end{aligned}
 \end{equation*}
 Summing this inequality from $t = 1$ to $T$, we obtain
  \begin{equation*}
   \begin{aligned}
     \sum_{t=1}^{T}[J(\boldsymbol{x}^{\ast})-J(\boldsymbol{x}^{adv}_{t})] & \leq \frac{\mu}{1-\mu}[J(\boldsymbol{x}^{adv}_{T})-J(\boldsymbol{x}^{adv}_{0})] + \frac{1-\mu}{2\alpha}[\|\boldsymbol{x}^{\ast}-(\boldsymbol{x}^{adv}_{1}+\mathbf{p}_{1})\|^2\\
   & - \|\boldsymbol{x}^{\ast}-(\boldsymbol{x}^{adv}_{T+1}+\mathbf{p}_{T+1})\|^2] + \sum_{t=1}^{T}\frac{\alpha}{2(1-\mu)}\|\boldsymbol{g}_t\|^2,\\
   \end{aligned}
 \end{equation*}
 i.e.,
 \begin{equation}\label{c}
   \begin{aligned}
     \sum_{t=1}^{T}[J(\boldsymbol{x}^{\ast})-J(\boldsymbol{x}^{adv}_{t})] & \leq \frac{\mu}{1-\mu}[J(\boldsymbol{x}^{adv}_{T})-J(\boldsymbol{x}^{adv}_{0})] + \frac{(1-\mu)\sqrt{T}}{2C}\|\boldsymbol{x}^{\ast}-(\boldsymbol{x}^{adv}_{1}+\mathbf{p}_{1})\|^2\\
   & + \frac{C}{2(1-\mu)\sqrt{T}}\sum_{t=1}^{T}\|\boldsymbol{g}_t\|^2.\\
   \end{aligned}
 \end{equation}
 Note
  \begin{equation}\label{d}
   \begin{aligned}
     \sum_{t=1}^{T}\|\boldsymbol{g}_t\|^2\leq TG^2.
   \end{aligned}
 \end{equation}
 Since $\mathbf{Q}$ is a bounded set, there exists a positive number $M > 0$ such that
  \begin{equation}\label{e}
   \begin{aligned}
     \|\boldsymbol{x}^{\ast}-(\boldsymbol{x}^{adv}_{t+1}+\mathbf{p}_{t+1}) \|^2 \leq M, \forall t \geq 0
   \end{aligned}
 \end{equation}
 From (\ref{c}), (\ref{d}) and (\ref{e}), we have
 \begin{equation*}
   \begin{aligned}
     \sum_{t=1}^{T}[J(\boldsymbol{x}^{\ast})-J(\boldsymbol{x}^{adv}_{t})] & \leq \frac{\mu}{1-\mu}[J(\boldsymbol{x}^{adv}_{T})-J(\boldsymbol{x}^{adv}_{0})] + \frac{(1-\mu)\sqrt{T}M}{2C} + \frac{C \sqrt{T} G^2}{2(1-\mu)}.
   \end{aligned}
 \end{equation*}
 Using the concavity of $J(\boldsymbol{x}^{adv})$, we obtain
 \begin{equation*}
   \begin{aligned}
     J(\boldsymbol{x}^{\ast})-J(\frac{1}{T}\sum_{t=1}^{T}\boldsymbol{x}^{adv}_{t}) & \leq \frac{\mu}{(1-\mu)T}[J(\boldsymbol{x}^{adv}_{T})-J(\boldsymbol{x}^{adv}_{0})] + \frac{(1-\mu)M}{2C \sqrt{T}} + \frac{CG^2}{2(1-\mu)\sqrt{T}}.
   \end{aligned}
 \end{equation*}
\end{proof}
This completes the proof of Theorem \ref{thm:bigtheorem}.

Note that our AdaMI-FGM (\ref{algorithm2}) uses the same adaptive step-size strategy as the AdaGrad algorithm in Mukkamala \& Hein (2017). In Mukkamala \& Hein (2017), the regret analysis in an online setting is provided. To analyze convergence of our AdaMI-FGM (\ref{algorithm2}) with a constant learning rate, we first introduce an important lemma from Mukkamala \& Hein (2017). In the following, for the diagonal matrix sequence $\{V_t\}^T_{t=1}$, we use $v_{t,i}$ to denote the $i$-th element in the diagonal of $V_t$, $g_{t,i}$ to denote the $i$-th element of $\boldsymbol{g}_t$ and $\boldsymbol{g}_{1:T, i}=[g_{1,i}, \cdots, g_{T,i}]$ to denote the vector obtained by concatenating the $i$-th element of the gradient sequence $\{\boldsymbol{g}_t\}^T_{t=1}$.

\begin{lemma}
\label{lem:usefullemma1}
(Mukkamala \& Hein, 2017) Suppose that $1-\frac{1}{t} \leq \beta_t \leq 1-\frac{\gamma}{t}$ for some $0<\gamma\leq 1$ and each $t\geq 1$ in AdaMI-FGM (\ref{algorithm2}). Then
 \begin{equation*}
   \begin{aligned}
     \sum_{i=1}^{d}\sum_{t=1}^{T}\frac{g^2_{t,i}}{\sqrt{tv_{t,i}}+\delta} \leq \sum_{i=1}^{d}\frac{2(2-\gamma)}{\gamma}(\sqrt{Tv_{T,i}}+\delta)
   \end{aligned}
 \end{equation*}
\end{lemma}

\begin{theorem}
\label{thm:bigtheorem1}
Let $\boldsymbol{x}^{\ast} $ be a solution of problem (\ref{adv-optimization}). Assume that $\mathbf{Q}$ is bounded and $0 \leq \mu <1 $. Let $\beta_t = 1-\frac{\gamma}{t}$ with some $ 0 < \gamma \leq 1$. Let $\{ \boldsymbol{x}^{adv}_t\}_{t=1}^{\infty} $ be generated by AdaMI-FGM (\ref{algorithm2}).
Set\\
\begin{gather*}
  \mathbf{p}_t=\frac{\mu}{1-\mu}(\boldsymbol{x}^{adv}_t-\boldsymbol{x}^{adv}_{t-1}) \ and \ \alpha \equiv \frac{C}{\sqrt{T}}.
\end{gather*}
Then we have
$$
  \boldsymbol{x}^{adv}_{t+1}+\mathbf{p}_{t+1}=P_{\mathbf{Q}}[\boldsymbol{x}^{adv}_t+\mathbf{p}_{t}+\frac{\alpha \hat{V}_t^{-1}}{1-\mu}\boldsymbol{g}_t]$$
  and
  $$J(\boldsymbol{x}^\ast)-J(\boldsymbol{\bar{x}}_T^{adv}) \leq O(\frac{\sum_{i=1}^{d}\|\boldsymbol{g}_{1:T,i}\|}{T}),$$
where $\boldsymbol{\bar{x}}_T^{adv} = \frac{1}{T}\sum_{t=1}^{T}\boldsymbol{x}^{adv}_t$. This convergence rate can be considerably better than the convergence rate $O(\frac{d G_{\infty}}{\sqrt{T}})$  of SGD when $\sum_{i=1}^{d}\| \boldsymbol{g}_{1:T,i}\| \ll d\sqrt{T}$.
\end{theorem}

\begin{proof}
 According to Lemma \ref{lem:usefullemma}, $\forall \boldsymbol{x}^{adv} \in \mathbf{Q}$, we have
 \begin{gather*}
   \langle \boldsymbol{x}^{adv}_{t+1}-\boldsymbol{x}^{adv}_{t}-\mu(\boldsymbol{x}^{adv}_{t}-\boldsymbol{x}^{adv}_{t-1})-\alpha \hat{V}_t^{-1} \boldsymbol{g}_t, \boldsymbol{x}^{adv}_{t+1}-\boldsymbol{x}^{adv}\rangle \leq 0.
 \end{gather*}
 This is
  \begin{gather*}
   \langle \frac{1}{1-\mu}(\boldsymbol{x}^{adv}_{t+1}-\boldsymbol{x}^{adv}_{t})-\mathbf{p}_{t}-\frac{\alpha \hat{V}_t^{-1}}{1-\mu}\boldsymbol{g}_t, \boldsymbol{x}^{adv}_{t+1}-\boldsymbol{x}^{adv}\rangle \leq 0,
 \end{gather*}
 i.e.,
  \begin{equation}\label{f}
   \langle \boldsymbol{x}^{adv}_{t+1}+\mathbf{p}_{t+1}-(\boldsymbol{x}^{adv}_{t}+\mathbf{p}_{t})-\frac{\alpha \hat{V}_t^{-1}}{1-\mu}\boldsymbol{g}_t, \boldsymbol{x}^{adv}_{t+1}-\boldsymbol{x}^{adv}\rangle \leq 0.
 \end{equation}
 Specifically,
 \begin{equation}\label{g}
   \langle \boldsymbol{x}^{adv}_{t+1}+\mathbf{p}_{t+1}-(\boldsymbol{x}^{adv}_{t}+\mathbf{p}_{t})-\frac{\alpha \hat{V}_t^{-1}}{1-\mu}\boldsymbol{g}_t, \frac{\mu(\boldsymbol{x}^{adv}_{t+1}-\boldsymbol{x}^{adv}_t)}{1-\mu}\rangle \leq 0.
 \end{equation}
 From (\ref{f}) and (\ref{g}), we have
  \begin{equation*}
   \langle \boldsymbol{x}^{adv}_{t+1}+\mathbf{p}_{t+1}-(\boldsymbol{x}^{adv}_{t}+\mathbf{p}_{t})-\frac{\alpha \hat{V}_t^{-1}}{1-\mu}\boldsymbol{g}_t, \boldsymbol{x}^{adv}_{t+1}+\mathbf{p}_{t+1}-\boldsymbol{x}^{adv}\rangle \leq 0.
 \end{equation*}
 Using Lemma \ref{lem:usefullemma}, we have
  \begin{equation*}
   \boldsymbol{x}^{adv}_{t+1}+\mathbf{p}_{t+1}=P_{\mathbf{Q}}[\boldsymbol{x}^{adv}_t+\mathbf{p}_{t}+\frac{\alpha \hat{V}_t^{-1}}{1-\mu}\boldsymbol{g}_t].
 \end{equation*}
 Then
 \begin{equation*}
   \begin{aligned}
   \|\boldsymbol{x}^{\ast}-(\boldsymbol{x}^{adv}_{t+1}+\mathbf{p}_{t+1}) \|^2_{\hat{V}_t} &\leq \|\boldsymbol{x}^{\ast}-(\boldsymbol{x}^{adv}_{t}+\mathbf{p}_{t}) -\frac{\alpha\hat{V}_t^{-1}}{1-\mu}\boldsymbol{g}_t\|^2_{\hat{V}_t} \\
   & = \|\boldsymbol{x}^{\ast}-(\boldsymbol{x}^{adv}_{t}+\mathbf{p}_{t})\|^2_{\hat{V}_t} + \| \frac{\alpha}{1-\mu}\boldsymbol{g}_t\|^2_{\hat{V}_t^{-1}} - 2\langle \frac{\alpha}{1-\mu}\boldsymbol{g}_t , \boldsymbol{x}^{\ast}-\boldsymbol{x}^{adv}_{t}\rangle \\
   & - 2\langle \frac{\alpha \mu}{(1-\mu)^2}\boldsymbol{g}_t , \boldsymbol{x}^{adv}_{t-1}-\boldsymbol{x}^{adv}_{t}\rangle.\\
   \end{aligned}
 \end{equation*}
 According to the properties of concave function,
 \begin{equation*}
   \langle \boldsymbol{g}_t , \boldsymbol{x}^{\ast}-\boldsymbol{x}^{adv}_{t}\rangle \geq J(\boldsymbol{x}^{\ast})-J(\boldsymbol{x}^{adv}_{t}), \langle \boldsymbol{g}_t , \boldsymbol{x}^{adv}_{t-1}-\boldsymbol{x}^{adv}_{t}\rangle \geq J(\boldsymbol{x}^{adv}_{t-1})-J(\boldsymbol{x}^{adv}_{t}).
 \end{equation*}
 Then
 \begin{equation*}
   \begin{aligned}
   \|\boldsymbol{x}^{\ast}-(\boldsymbol{x}^{adv}_{t+1}+\mathbf{p}_{t+1}) \|^2_{\hat{V}_t} & \leq \|\boldsymbol{x}^{\ast}-(\boldsymbol{x}^{adv}_{t}+\mathbf{p}_{t})\|^2_{\hat{V}_t} + \frac{\alpha^2}{(1-\mu)^2}\|\boldsymbol{g}_t\|^2_{\hat{V}_t^{-1}} \\
   & + \frac{2\alpha}{1-\mu}[J(\boldsymbol{x}^{adv}_{t})-J(\boldsymbol{x}^{\ast})] + \frac{2\alpha \mu}{(1-\mu)^2}[J(\boldsymbol{x}^{adv}_{t})-J(\boldsymbol{x}^{adv}_{t-1})].\\
   \end{aligned}
 \end{equation*}
Rearrange the inequality, we have
 \begin{equation*}
   \begin{aligned}
    \frac{2\alpha}{1-\mu}[J(\boldsymbol{x}^{\ast})-J(\boldsymbol{x}^{adv}_{t})] & \leq \frac{2\alpha \mu}{(1-\mu)^2}[J(\boldsymbol{x}^{adv}_{t})-J(\boldsymbol{x}^{adv}_{t-1})] +\| \boldsymbol{x}^{\ast}-(\boldsymbol{x}^{adv}_{t}+\mathbf{p}_{t})\|^2_{\hat{V}_t}\\
   & - \|\boldsymbol{x}^{\ast}-(\boldsymbol{x}^{adv}_{t+1}+\mathbf{p}_{t+1}) \|^2_{\hat{V}_t} + \frac{\alpha^2}{(1-\mu)^2}\|\boldsymbol{g}_t\|^2_{\hat{V}_t^{-1}},\\
   \end{aligned}
 \end{equation*}
 i.e.,
 \begin{equation*}
   \begin{aligned}
    J(\boldsymbol{x}^{\ast})-J(\boldsymbol{x}^{adv}_{t}) & \leq \frac{\mu}{1-\mu}[J(\boldsymbol{x}^{adv}_{t})-J(\boldsymbol{x}^{adv}_{t-1})] + \frac{1-\mu}{2\alpha}[\|\boldsymbol{x}^{\ast}-(\boldsymbol{x}^{adv}_{t}+\mathbf{p}_{t})\|^2_{\hat{V}_t}\\
   & - \|\boldsymbol{x}^{\ast}-(\boldsymbol{x}^{adv}_{t+1}+\mathbf{p}_{t+1}) \|^2_{\hat{V}_t}] + \frac{\alpha}{2(1-\mu)}\|\boldsymbol{g}_t\|^2_{\hat{V}_t^{-1}}.\\
   \end{aligned}
 \end{equation*}
 Summing this inequality from $t = 1$ to $T$, we obtain
  \begin{equation*}
   \begin{aligned}
     \sum_{t=1}^{T}[J(\boldsymbol{x}^{\ast})-J(\boldsymbol{x}^{adv}_{t})] & \leq \frac{\mu}{1-\mu}[J(\boldsymbol{x}^{adv}_{T})-J(\boldsymbol{x}^{adv}_{0})] + \frac{1-\mu}{2\alpha}[\|\boldsymbol{x}^{\ast}-(\boldsymbol{x}^{adv}_{1}+\mathbf{p}_{1})\|^2_{\hat{V}_1}\\
   & - \|\boldsymbol{x}^{\ast}-(\boldsymbol{x}^{adv}_{T+1}+\mathbf{p}_{T+1})\|^2_{\hat{V}_T}] + \frac{\alpha}{2(1-\mu)}\sum_{t=1}^{T}\|\boldsymbol{g}_t\|^2_{\hat{V}_t^{-1}}\\
   & + \frac{1-\mu}{2\alpha}\sum_{i=1}^{d} \sum_{t=2}^{T}[x^{\ast}_i-(x^{adv}_{t,i}+p_{t,i})]^2(\hat{v}_{t,i}-\hat{v}_{t-1,i}),
   \end{aligned}
 \end{equation*}
 i.e.,
 \begin{equation}\label{h}
   \begin{aligned}
     \sum_{t=1}^{T}[J(\boldsymbol{x}^{\ast})-J(\boldsymbol{x}^{adv}_{t})] & \leq \frac{\mu}{1-\mu}[J(\boldsymbol{x}^{adv}_{T})-J(\boldsymbol{x}^{adv}_{0})] + \frac{(1-\mu)\sqrt{T}}{2C}\|\boldsymbol{x}^{\ast}-(\boldsymbol{x}^{adv}_{1}+\mathbf{p}_{1})\|^2_{\hat{V}_1}\\
   & + \frac{C}{2(1-\mu)\sqrt{T}}\sum_{i=1}^{d}\sum_{t=1}^{T}\frac{\sqrt{t} g_{t,i}^2}{\sqrt{tv_{t,i}}+\delta} + \frac{(1-\mu)\sqrt{T}}{2C}\sum_{i=1}^{d} \sum_{t=2}^{T}[x^{\ast}_i-(x_{t,i}+p_{t,i})]^2(\hat{v}_{t,i}-\hat{v}_{t-1,i})\\
   & \leq \frac{\mu}{1-\mu}[J(\boldsymbol{x}^{adv}_{T})-J(\boldsymbol{x}^{adv}_{0})] + \frac{(1-\mu)\sqrt{T}}{2C}\sum_{i=1}^{d}[x^{\ast}_i-x_{t,i}+p_{t,i})]^2\hat{v}_{1,i}\\
   & + \frac{C}{2(1-\mu)}\sum_{i=1}^{d}\sum_{t=1}^{T}\frac{ g_{t,i}^2}{\sqrt{tv_{t,i}}+\delta} + \frac{(1-\mu)\sqrt{T}}{2C}\sum_{i=1}^{d} \sum_{t=2}^{T}[x^{\ast}_i-(x_{t,i}+p_{t,i})]^2(\hat{v}_{t,i}-\hat{v}_{t-1,i}).\\
   \end{aligned}
 \end{equation}
 Using Lemma \ref{lem:usefullemma1}, we have
  \begin{equation}\label{i}
   \begin{aligned}
     \frac{C}{2(1-\mu)}\sum_{i=1}^{d}\sum_{t=1}^{T}\frac{g^2_{t,i}}{\sqrt{tv_{t,i}}+\delta} \leq \frac{C}{2(1-\mu)}\sum_{i=1}^{d}\frac{2(2-\gamma)}{\gamma}(\sqrt{Tv_{T,i}}+\delta) =
     \frac{C(2-\gamma)}{\gamma(1-\mu)}\sum_{i=1}^{d}(\|\boldsymbol{g}_{1:T, i}\|+\delta).
   \end{aligned}
 \end{equation}
 Since $\mathbf{Q}$ is a bounded set, there exists a positive number $M_\infty > 0$ such that
  \begin{equation}\label{j}
   \begin{aligned}
     [x^{\ast}_i-(x_{t,i}+p_{t,i})]^2 \leq M_\infty, \ \forall t \geq 0, 1 \leq i \leq d.
   \end{aligned}
 \end{equation}
 From (\ref{h}), (\ref{i}), (\ref{j}) and Assumption 1, we have
 \begin{equation*}
   \begin{aligned}
     \sum_{t=1}^{T}[J(\boldsymbol{x}^{\ast})-J(\boldsymbol{x}^{adv}_{t})] & \leq \frac{\mu}{1-\mu}[J(\boldsymbol{x}^{adv}_{T})-J(\boldsymbol{x}^{adv}_{0})] + \frac{(1-\mu)\sqrt{T}M_\infty}{2C}\sum_{i=1}^{d}\hat{v}_{1,i} \\
     & + \frac{C(2-\gamma)}{\gamma(1-\mu)}\sum_{i=1}^{d}(\|\boldsymbol{g}_{1:T, i}\|+\delta)+\frac{(1-\mu)\sqrt{T}M_\infty}{2C}\sum_{i=1}^{d} \sum_{t=2}^{T}(\hat{v}_{t,i}-\hat{v}_{t-1,i})\\
     & = \frac{\mu}{1-\mu}[J(\boldsymbol{x}^{adv}_{T})-J(\boldsymbol{x}^{adv}_{0})] + \frac{(1-\mu)\sqrt{T}M_\infty}{2C}\sum_{i=1}^{d}\hat{v}_{1,i} \\
     & + \frac{C(2-\gamma)}{\gamma(1-\mu)}\sum_{i=1}^{d}(\|\boldsymbol{g}_{1:T, i}\|+\delta)+\frac{(1-\mu)\sqrt{T}M_\infty}{2C}\sum_{i=1}^{d} (\hat{v}_{T,i}-\hat{v}_{1,i})\\
     & = \frac{\mu}{1-\mu}[J(\boldsymbol{x}^{adv}_{T})-J(\boldsymbol{x}^{adv}_{0})] + \frac{C(2-\gamma)}{\gamma(1-\mu)}\sum_{i=1}^{d}(\|\boldsymbol{g}_{1:T, i}\|+\delta)+\frac{(1-\mu)\sqrt{T}M_\infty}{2C}\sum_{i=1}^{d} \hat{v}_{T,i}\\
     & = \frac{\mu}{1-\mu}[J(\boldsymbol{x}^{adv}_{T})-J(\boldsymbol{x}^{adv}_{0})] + \frac{C(2-\gamma)}{\gamma(1-\mu)}\sum_{i=1}^{d}(\|\boldsymbol{g}_{1:T, i}\|+\delta)\\
     & +\frac{(1-\mu)M_\infty}{2C}\sum_{i=1}^{d} (\|\boldsymbol{g}_{1:T, i}\|+\delta)\\
   \end{aligned}
 \end{equation*}
 By concavity of $J(\boldsymbol{x}^{adv})$, we obtain
 \begin{equation*}
   \begin{aligned}
     J(\boldsymbol{x}^{\ast})-J(\frac{1}{T}\sum_{t=1}^{T}\boldsymbol{x}^{adv}_{t}) & \leq \frac{\mu}{(1-\mu)T}[J(\boldsymbol{x}^{adv}_{T})-J(\boldsymbol{x}^{adv}_{0})]+\frac{C(2-\gamma)}{\gamma(1-\mu)T}\sum_{i=1}^{d}(\|\boldsymbol{g}_{1:T, i}\|+\delta)\\
     & +\frac{(1-\mu)M_\infty}{2CT}\sum_{i=1}^{d} (\|\boldsymbol{g}_{1:T, i}\|+\delta)
   \end{aligned}
 \end{equation*}
\end{proof}
This completes the proof of Theorem \ref{thm:bigtheorem1}.

\end{document}